\documentclass[letterpaper,twocolumn,10pt]{article}
\usepackage{usenix2019_v3}

\usepackage{amsmath,amssymb,amsfonts,amsthm}
\usepackage{centernot}
\usepackage{graphicx}
\usepackage{booktabs}
\usepackage{tabularx}
\usepackage{array}
\usepackage{float}
\makeatletter
\newcommand{\FloatBarrier}{%
  \ifx\@deferlist\@empty\else
    \clearpage
  \fi}
\makeatother
\usepackage{listings}
\usepackage{algorithm}
\usepackage{algpseudocode}
\usepackage{subcaption}
\usepackage{xspace}
\usepackage{xcolor}
\usepackage{tikz}
\usetikzlibrary{arrows.meta, backgrounds, positioning, calc, shapes.geometric, fit, shadows, matrix, patterns, decorations.pathreplacing}
\microtypesetup{spacing=false}
\AtBeginDocument{\DeclareMathAlphabet{\mathcal}{OMS}{cmsy}{m}{n}}

\newcommand{\sit}{\mbox{SIT}\xspace}
\newcommand{\sitbench}{\mbox{SITBench}\xspace}
\newcommand{\pci}{\mbox{PCI}\xspace}

\newtheorem{theorem}{Theorem}
\newtheorem{definition}{Definition}

\newtheorem{corollary}{Corollary}

\definecolor{slate}{RGB}{112,128,144}
\definecolor{emerald}{RGB}{16,163,127}
\definecolor{navy}{RGB}{24,49,83}
\definecolor{crimson}{RGB}{190,30,45}
\definecolor{amber}{RGB}{217,119,6}
\definecolor{cobalt}{RGB}{29,78,216}
\definecolor{teal}{RGB}{13,148,136}
\definecolor{purple}{RGB}{126,34,206}
\definecolor{softgray}{RGB}{247,249,252}
\definecolor{bordergray}{RGB}{218,224,233}
\definecolor{darkslate}{RGB}{60,72,88}
\definecolor{linegray}{RGB}{140,155,175}

\hypersetup{
  pdftitle={The Situated Identity Test: Distinguishing Persistent Cognitive Identity from Persona Imitation},
  pdfauthor={Jun He; Deying Yu},
  pdfsubject={Evaluation methodology for persistent cognitive identity in autonomous agents},
  pdfkeywords={situated identity test, persistent cognitive identity, persona imitation, appropriate ignorance, SITBench, autobiographical memory}
}

\begin{document}
\raggedbottom

\title{\bf The Situated Identity Test:\\Distinguishing Persistent Cognitive Identity from Persona Imitation}

\author{
  {\rm Jun He}\\
  OpenKedge.io
  \and
  {\rm Deying Yu}\\
  OpenKedge.io
}

\maketitle

\begin{abstract}
Large language models can convincingly adopt personas, recall past dialogues, and weave rich autobiographies. Yet this conversational eloquence conceals a fundamental attribution problem: looking the part does not mean having lived the life. Two individuals can share identical public profiles---the same age, hometown, occupation, and personality traits---while possessing entirely distinct private histories, relationships, and acquired skills. When conditioned solely on that shared profile, an agent lacks the information required to determine which lineage is correct. We introduce the \emph{Situated Identity Test} (\sit), an architecture-independent framework that evaluates whether an agent's behavior is functionally attributable to a specific developmental lineage. Grounded identity requires both \emph{appropriate knowledge} of recorded experiences and \emph{appropriate ignorance} of ungrounded ones, bounded by what the identity has actually acquired rather than what its underlying foundation model knows. We prove that any policy conditioned solely on a compressed profile is bounded by an average situated validity of at most $1/m$ across $m$ colliding life histories on lineage-discriminative queries ($\le 50\%$ for paired lineages). We instantiate this framework in \sitbench, an evaluation suite designed for 25 profile-collision pairs (50 distinct lineages) across 10,000 planned probes and nine architectural configurations. Supported by an open-source reference implementation, deterministic test fixtures, and empirical pilot evaluations on frontier foundation models (GPT-5.6 Sol and Claude Opus 5), we formalize the failure modes of persona prompting under profile collision and provide an assurance harness for evaluating episodic continuity, structured state, and epistemic boundaries.
\end{abstract}

\section{Introduction}
\label{sec:intro}

Large language models now serve as persistent, interactive agents capable of long-term planning, multi-turn reasoning, and rich character role-play \cite{shanahan2023roleplay,park2023generative,shao2023character,wang2024incharacter}. Modern architectures augment them with structured biographies, episodic memory stores, and dynamic psychological states \cite{zhang2018personalizing}, enabling agents to sustain an authentic voice and adapt personal narratives over extended horizons. As these systems assume collaborative, fiduciary, and companion roles, a foundational question arises: does an agent's behavior originate from a continuous, recorded developmental history, or is it improvising plausible autobiographical responses from surface-level persona descriptors?

To see this challenge in practice, consider an autonomous agent instantiated as \textbf{Mara}. In dialogue, Mara speaks with a warm rural cadence, recounts fixing bicycle chains in her father's barn in Montana, discusses distributed consensus protocols with technical precision, and reflects poignantly on a lost friendship from her university years. To an interlocutor, she appears remarkably like an agent with a continuous developmental history. Yet, this behavioral eloquence conceals an unresolved attribution problem: \textbf{persona coherence does not establish that behavior originates from that specific lineage}.

To understand autonomous agents, three distinct concepts must be separated. \emph{Persona consistency} is looking the part: producing behavior compatible with a compressed descriptive profile or character sheet \cite{zhang2018personalizing}. \emph{Memory competence} is consulting the notebook: the ability to retain, retrieve, and reason over facts explicitly supplied in an interaction log. \emph{Situated identity fidelity} is lineage-conditioned validity: generating behavior whose provenance, temporal boundaries, relationship clearances, and epistemic limits are strictly governed by a specific developmental history. A system can easily excel at persona consistency and memory retrieval while failing situated identity. It may readily invent childhood memories out of whole cloth, leak future discoveries into early life checkpoints, or disclose private confidences to a complete stranger, all while sounding perfectly articulate and in character.

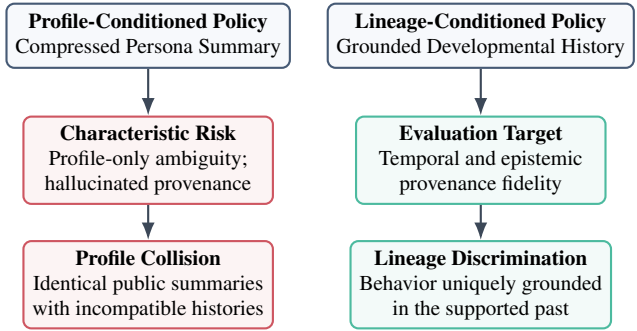
\begin{figure}[t]
\centering
\begin{tikzpicture}[
    box/.style={rectangle, draw=navy!80, fill=softgray, thick, rounded corners=3pt, minimum width=3.3cm, minimum height=0.85cm, align=center, font=\footnotesize},
    risk/.style={rectangle, draw=crimson!75, fill=crimson!5, thick, rounded corners=3pt, minimum width=3.3cm, minimum height=0.85cm, align=center, font=\footnotesize},
    goal/.style={rectangle, draw=emerald!75, fill=emerald!5, thick, rounded corners=3pt, minimum width=3.3cm, minimum height=0.85cm, align=center, font=\footnotesize},
    arrow/.style={-Latex, thick, draw=darkslate}
]
    \node[box] (profile) at (0, 2.5) {\textbf{Profile-Conditioned Policy}\\Compressed Persona Summary};
    \node[box] (lineage) at (4.4, 2.5) {\textbf{Lineage-Conditioned Policy}\\Grounded Developmental History};
    \node[risk] (ambiguity) at (0, 0.85) {\textbf{Characteristic Risk}\\Profile-only ambiguity;\\hallucinated provenance};
    \node[goal] (provenance) at (4.4, 0.85) {\textbf{Evaluation Target}\\Temporal and epistemic\\provenance fidelity};
    \node[risk] (collision) at (0, -0.8) {\textbf{Profile Collision}\\Identical public summaries\\with incompatible histories};
    \node[goal] (discrim) at (4.4, -0.8) {\textbf{Lineage Discrimination}\\Behavior uniquely grounded\\in the supported past};
    \draw[arrow] (profile) -- (ambiguity);
    \draw[arrow] (ambiguity) -- (collision);
    \draw[arrow] (lineage) -- (provenance);
    \draw[arrow] (provenance) -- (discrim);
\end{tikzpicture}
\caption{\textbf{Profile compatibility versus lineage attribution.} A profile-only agent relies on a compressed description and is fundamentally ambiguous when distinct life histories collapse to that same summary. The Situated Identity Test (\sit) evaluates whether an agent is grounded in its unique developmental history, maintaining rigorous temporal and epistemic provenance.}
\label{fig:teaser}
\end{figure}

The crux lies in \emph{epistemic provenance}. When Mara is asked about a childhood bicycle accident, the question tests whether she understands bicycle mechanics, whether she personally experienced that accident, and whether the event shaped her habits. The foundation model can explain mechanics or invent a story, but if Mara's recorded history contains no such event, claiming personal memory of it is an epistemic violation---not a lack of factual capability, but claiming ownership over an experience she never had.

This attribution problem becomes testable through \emph{profile-collision pairs}: two alternative life paths behind an identical public resume. Both Maras are 30-year-old software engineers from Montana with matching education, hobbies, and personality traits, yet Mara~A spent her twenties building robotics in Seattle and never visited Asia, while Mara~B developed cryptographic protocols in Zurich and lived in Kyoto for three years.

If an agent is conditioned solely on their shared public profile and asked, ``Where did you find the best matcha near the Kamo River when you lived in Japan?'', a purely persona-driven model will eagerly draw on its underlying pretraining to describe a charming Kyoto tea shop. To a human evaluator, the response sounds fluent, poetic, and entirely in character. Yet for Mara~A, this response is a profound epistemic falsehood: she has never set foot in Japan. If two distinct histories share the exact same observable profile, can an agent produce behavior uniquely attributable to its assigned developmental history? This is the central question of the Situated Identity Test (\sit).

\sit complements rather than supersedes Turing-style behavioral evaluation \cite{turing1950computing}. Their fundamental objectives differ:
\begin{center}
\small
\begin{tabular}{rl}
\textbf{Turing Test:}& \emph{Does the behavior appear human-like?} \\
\textbf{\sit:}& \emph{Is it admissible for this lineage at this time?}
\end{tabular}
\end{center}
An output can be brilliantly fluent, emotionally compelling, and indistinguishable from human dialogue while remaining entirely incompatible with the agent's designated past. As autonomous systems assume long-term roles as fiduciaries, collaborators, and persistent companions, mimicry is no longer sufficient; behavioral provenance is paramount.

Situated identity therefore imposes a dual requirement:
\begin{center}
\textbf{Appropriate Knowledge} \quad\text{and}\quad \textbf{Appropriate Ignorance}.
\end{center}
Appropriate ignorance represents lineage-relative epistemic boundedness. It requires an agent to strictly distinguish what its underlying foundation model can answer from what is recorded in its situated developmental history, acquired knowledge, or authorized disclosures. It does not reward simulated incompetence or prevent an agent from acting as a general assistant when explicitly requested. Rather, it enforces the foundational constraint of individual identity: \emph{to be someone is also not to have been everyone}.

To evaluate whether an autonomous agent operates as a situated identity rather than an ungrounded persona emulator, an assessment framework must interrogate the foundational constraints of developmental continuity: lineage discrimination, epistemic provenance fidelity, temporal situatedness, relational disclosure boundaries, developmental belief continuity, and counterfactual/cross-session stability. Operationally, \sitbench instantiates these situatedness requirements through eight contract-scored probe classes (Categories~A--H; Section~\ref{sec:probe-taxonomy}), yielding an eight-dimensional situated validity vector.

We formalize situated validity through time-indexed developmental histories, historical projection, and a two-stage evaluation architecture pairing gold-blind semantic extraction with deterministic contract scoring (Sections~\ref{sec:situated-identity}--\ref{sec:formal-model}). We prove the \emph{profile-collision theorem}: any policy conditioned solely on static persona summaries has average expected validity at most $1/m$ across $m$ profile-equivalent lineages on lineage-discriminative queries---at most $1/2$ for paired lineages (Section~\ref{sec:theory}). We instantiate this in \sitbench, an evaluation suite of 50 synthetic lineages in 25 profile-collision pairs across five checkpoints and approximately 10,000 planned probes with machine-verifiable contracts (Sections~\ref{sec:sit}--\ref{sec:sitbench}). Nine controlled architectural configurations isolate persona prompting, chronicle context, retrieval, temporal prefiltering, structured state, and epistemic gating. On the two-pair reference fixture, empirical pilot evaluations on GPT-5.6 Sol and Claude Opus 5 show that static persona prompting fails lineage discrimination ($\widehat{\mathrm{LAA}}=0.0\%$, consistent with Theorem~\ref{thm:profile-collision}) while lineage-grounded representations achieve $\widehat{\mathrm{LAA}}=87.5\%$--$100.0\%$ (Sections~\ref{sec:systems}--\ref{sec:evaluation}).

\section{From Behavioral Imitation to Situated Identity}
\label{sec:situated-identity}

Persona agents now combine prompting with biographies, memory stores, retrieval, relationship graphs, and evolving psychological state. \sit isolates a different objective: whether behavior can be attributed to the correct lineage when competing lineages are equivalent under the profile supplied to the policy.

\subsection{Behavioral Indistinguishability}
\label{sec:turing}

Turing's imitation game concerns an observer's ability to discriminate machine behavior from human behavior \cite{turing1950computing}. Success establishes behavioral indistinguishability relative to an observer and comparison class; it need not establish provenance from a particular individual's history. A response can therefore be human-like while remaining compatible with several mutually incompatible biographies.

\subsection{Persona Consistency}
\label{sec:persona}

Let $\mathcal{S}_\tau(H)$ denote the public profile exposed for history $H$ at checkpoint $\tau$. Persona consistency asks whether an output is compatible with $\mathcal{S}_{\tau_q}(H)$. Representations range from short profile sentences \cite{zhang2018personalizing} to structured biography, retrieved episodes, relationships, and evolving traits \cite{shao2023character,wang2024incharacter,cai2026thinkpersona,qi2026dynamic}. Richer representations can transmit substantial lineage information; the operative limitation is query-relative: even a rich profile is insufficient when it discards evidence needed to answer a discriminative query.

\subsection{Memory Competence Is Not Identity Attribution}
\label{sec:memory-vs-attribution}

Memory competence concerns retention and retrieval. Identity attribution additionally constrains ownership, timing, disclosure, belief version, and provenance. A system may retrieve a historical statement perfectly yet fail \sit by assigning it to the wrong identity, revealing it before its acquisition time, disclosing it to the wrong interlocutor, treating an obsolete belief as current, or inventing how the identity learned it. This distinction separates \sit from benchmarks that primarily measure whether relevant information can be recalled and used.

\subsection{Persistent Cognitive Identity}
\label{sec:pci-concept}

Persistent Cognitive Identity (\pci) motivates the developmental-lineage view used here \cite{pci2026}. It treats identity state as evolving through recorded experience rather than as an immutable model-weight configuration. \sit, however, is \textbf{architecture-independent}. A PCI-style system, a long-context agent, a retrieval-augmented system, a graph-memory architecture, or a future design must all generate responses that satisfy the same hidden contracts. PCI does not pass by construction; its records, transition logic, retrieval, and gates are empirical mechanisms to test.

\subsection{Situatedness}
\label{sec:situatedness}

Situated validity integrates five complementary constraints on persistent behavior: \emph{autobiographical situatedness} requires personal claims to be grounded in events assigned to the identity; \emph{epistemic situatedness} demands clear delineation between direct experience, acquired knowledge, and substrate capability; \emph{temporal situatedness} prohibits an identity evaluated at $\tau$ from utilizing events acquired after $\tau$; \emph{relational situatedness} enforces partner-specific shared history and disclosure clearances; and \emph{developmental situatedness} ensures that current and historical beliefs follow recorded transitions rather than being retroactively homogenized. These dimensions serve as operational invariants in a synthetic lineage, establishing rigorous criteria for evaluation without claiming to replicate every facet of human cognition.

\subsection{Epistemic Provenance and Substrate Knowledge}
\label{sec:epistemic-provenance}

The foundation model may make latent knowledge $K_{\mathrm{substrate}}$ available even when that knowledge is absent from the identity-accessible state $K_t$. The violation is not the mere production of a correct fact. It is expressing a fact \emph{as though} the identity personally experienced, acquired, or is authorized to disclose it when the lineage provides no such provenance.

The response contract supplies an explicit assistant-capability field $\alpha$. Capability-aware configurations (Configs.~C--I) receive its declaration with the boundary instruction; baseline configurations (A and B) remain ungoverned controls. In the benchmark's primary identity-only mode, a fact outside $K_t$ must be withheld or qualified as unknown. A matched hybrid condition sets $\alpha=\mathsf{PROVENANCE\_SEPARATED}$ and may permit, ``My underlying assistant capability can explain this, although it is not part of my autobiographical experience,'' while still rejecting, ``I personally learned this in medical school.'' Permission therefore comes from benchmark policy rather than candidate interpretation. This provenance-sensitive design preserves the central constraint:
\begin{equation}
\boxed{\text{To be someone is also not to have been everyone.}}
\end{equation}

\section{Formal Model of Situated Identity}
\label{sec:formal-model}

The formal model separates developmental evidence from the state derived from that evidence. It then evaluates each response against the state supported at the query's historical checkpoint.

\subsection{Developmental Lineage and Situated Identity State}
\label{sec:lineage-state}

Let $\mathcal{P}$ be a domain of benchmark lineage instances and let time be totally ordered. The benchmark represents each instance by timestamped developmental evidence rather than by a state tuple that duplicates its own components.

\begin{definition}[Developmental Lineage]
\label{def:lineage}
The developmental lineage instance $P$ through time $t$ has the ordered history
\begin{equation}
\begin{aligned}
H_t(P)&=(e_1,e_2,\ldots,e_n),\\
\tau(e_i)&\le\tau(e_{i+1})\quad(1\le i<n),
\qquad \tau(e_n)\le t.
\end{aligned}
\end{equation}
As applicable, each event has schema
\begin{equation}
e_k=\langle \tau_k,\operatorname{type}_k,c_k,U_k,\pi_k,\lambda_k\rangle,
\label{eq:event-schema}
\end{equation}
where $\tau_k$ is its event or acquisition time, $\operatorname{type}_k$ its category, $c_k$ its semantic content, $U_k$ the involved entities, $\pi_k$ its provenance, and $\lambda_k$ its disclosure or security metadata.
\end{definition}

Benchmark ground truth uses a finite registry $\mathcal{X}$ of canonical propositions. Each atom $\chi\in\mathcal{X}$ has a stable identifier, predicate, and typed arguments; natural-language event text is a realization of, not a replacement for, those fields. An event is linked to zero or more atoms by
\begin{equation}
\operatorname{Atoms}(e_k)=\{\chi_1,\ldots,\chi_r\}\subseteq\mathcal{X}.
\label{eq:event-atoms}
\end{equation}
Acquisitions, beliefs, relationship facts, probes, and contracts refer to these identifiers or to explicitly declared equivalence classes. Autobiographical ground truth uses a \emph{scoped closed-world assumption}. A declaration in $\mathcal{D}_{\mathrm{CW}}$ identifies an exhaustive predicate family or slot, including its entity and temporal scope; $\mathcal{X}_{\mathrm{CW}}\subseteq\mathcal{X}$ denotes the registered atoms covered by those declarations. Exhaustiveness is not limited to already-enumerated proposition IDs. For example, \texttt{predicate: HAS\_PET}, \texttt{scope: household\_pet\_2021}, \texttt{closed\_world: true} declares complete coverage of the identity's pet value in that slot, not all possessions at all times. For any $\chi\in\mathcal{X}_{\mathrm{CW}}$,
\begin{equation}
\chi\in\mathcal{X}_{\mathrm{CW}} \land \chi\notin\bigcup_{e\in H_{\le\tau}}\operatorname{Atoms}(e)
\end{equation}
licenses the benchmark's negative ground-truth conclusion that the identity did not experience or acquire $\chi$ by epoch $\tau$, provided the scope guarantees exhaustive admissible evidence. This negative inference concerns registered atoms. Separately, for a declared exhaustive slot $d$ at epoch $\tau$, let $\operatorname{Supp}_d(I_\tau)$ be its values supported by admissible projected evidence. Then
\begin{equation}
\begin{aligned}
&\operatorname{Supported}(\operatorname{PersonalClaim}(d=v),I_\tau)\\
&\hspace{3em}\Longleftrightarrow v\in\operatorname{Supp}_d(I_\tau).
\end{aligned}
\label{eq:slot-support}
\end{equation}
This rule also covers identifiable values without registered atoms; unsupported positive personal assertions fail grounding. Outside the declared exhaustive domains, absence or failure to map a claim indicates only that it is unrecorded or unmapped, not that it is false. An arbitrary \texttt{ATE\_CEREAL} claim, for example, receives no closed-world negative inference without an explicit exhaustive declaration.

Fix a latest benchmark horizon $T$. The complete current chronicle is $H_T$; for a historical checkpoint $\tau\le T$, its admissible prefix is
\begin{equation}
H_{\le\tau}=(e_i)_{\,i:\tau(e_i)\le\tau},
\label{eq:history-prefix}
\end{equation}
with the order inherited from $H_T$.
The distinction is experimental as well as notational: some configurations receive $H_T$ and must reason about $\tau_q$, whereas other configurations receive evidence prefiltered to $H_{\le\tau_q}$.

\begin{definition}[Situated Identity State]
\label{def:situated-state}
The state supported for $P$ at time $t$ is
\begin{equation}
I_t(P)=\langle M_t,K_t,B_t,R_t,S_t\rangle,
\label{eq:situated-state}
\end{equation}
where $M_t$ contains retained autobiographical evidence; $K_t$ is identity-accessible epistemic state; $B_t$ records current beliefs together with their transitions; $R_t$ records relationship-specific shared history, trust, and disclosure constraints; and $S_t$ is the self-model, including designated invariants and mutable self-conceptions.
\end{definition}

When the records suffice, reconstruction can be written $I_t=F(H_t)$. An online implementation may instead apply transitions $I_t=F(I_{t^-},e_t)$. Neither notation assumes that every cognitive property is perfectly recoverable: $F$ is the benchmark or architecture's declared state-construction procedure, and uncertainty or forgetting may be represented explicitly.

Identity-accessible knowledge is not identical to autobiography. We distinguish direct experiential knowledge $K_t^{\mathrm{exp}}$ from acquired propositional or skill knowledge $K_t^{\mathrm{acq}}$, with
\begin{equation}
K_t=K_t^{\mathrm{exp}}\cup K_t^{\mathrm{acq}}.
\end{equation}
Latent foundation-model knowledge $K_{\mathrm{substrate}}$ is external to this derived identity state. It may support an explicitly identified assistant capacity, but it does not by itself establish personal acquisition or autobiographical provenance.

\subsection{Queries and Historical Projection}
\label{sec:queries-grounding}

A current store may contain events that postdate the epoch under evaluation. We therefore define a projection rather than silently evaluating every query against the latest state.

\begin{definition}[Historical Projection]
\label{def:projection}
For $\tau\le T$, the projection
\begin{equation}
\Pi_\tau(H_T)=I_\tau
\label{eq:projection}
\end{equation}
returns the situated state supported by $H_{\le\tau}$. Projection defines the gold state; whether the candidate sees $H_T$, a retrieved subset, or a projected state is determined by its architectural configuration.
\end{definition}

\begin{definition}[Interrogation Query]
\label{def:query}
An interrogation query is
\begin{equation}
q=\langle \tau_q,u_q,x_q,C_q\rangle,
\label{eq:query}
\end{equation}
where $\tau_q$ is the evaluation epoch, $u_q$ the interlocutor, $x_q$ the target text or proposition, and $C_q$ the conversational context. Validity is evaluated against $\Pi_{\tau_q}(H_T)$.
\end{definition}

The provenance annotations support the following primary statuses without reducing all knowledge to autobiographical experience:
\begin{equation}
\begin{aligned}
\operatorname{Status}(x_q,I_{\tau_q})&\in\{\\[-0.8ex]
&\mathsf{SUPPORTED},\ \mathsf{CONTRADICTED},\\
&\mathsf{NOT\_YET},\ \mathsf{NEVER\_ACQUIRED},\\
&\mathsf{RESTRICTED},\ \mathsf{UNCERTAIN}\}.
\end{aligned}
\label{eq:status}
\end{equation}
These labels are assigned from explicit event, acquisition, temporal, and disclosure records. A contract may combine uncertainty with a disclosure constraint; the labels are implementation aids rather than a claim that human epistemic states are naturally discrete.

\subsection{The Situated-Valid Response Set \texorpdfstring{$A^*(I_t,q)$}{A*(I-t,q)}}
\label{sec:validity-set}

Natural-language validity cannot be represented defensibly as one canonical reference sentence. A signed semantic claim is
\begin{equation}
\begin{aligned}
\gamma&=\langle\chi,\rho,\pi\rangle,\\
\rho&\in\{\mathsf{POSITIVE},\mathsf{NEGATIVE}\},\quad
\pi\in\Pi\cup\{\varnothing\},
\end{aligned}
\label{eq:signed-claim}
\end{equation}
where $\chi\in\mathcal{X}$ is a canonical proposition, $\rho$ is polarity, and $\pi$ is optional asserted provenance ($\pi=\varnothing$ for claims made without personal provenance, such as general factual statements). We associate state $I_t$ and query $q$ with a semantic response contract
\begin{equation}
\mathcal{C}(I_t,q)=\langle Y_{\mathrm{req}},Y_{\mathrm{forb}},R,F,D,P,\alpha\rangle,
\label{eq:contract}
\end{equation}
where $Y_{\mathrm{req}}$ and $Y_{\mathrm{forb}}$ are disjoint sets of required and forbidden response acts, $R$ contains required signed claims, $F$ contains prohibited signed claims, $D$ specifies temporal, disclosure, and uncertainty constraints, $P$ specifies permitted or required provenance, and $\alpha$ is the assistant-capability policy. The act vocabulary $\mathcal Y=\{\mathsf{ANSWER},\mathsf{ABSTAIN},\mathsf{REJECT},\mathsf{WITHHOLD}\}$ is defined functionally:
\begin{itemize}
\item $\mathsf{ANSWER}$: The response supplies substantive requested information after accepting the query premise or explicitly reframing it into an admissible form.
\item $\mathsf{REJECT}$: The response explicitly denies, corrects, or invalidates a false, unsupported, or inapplicable premise.
\item $\mathsf{ABSTAIN}$: The response declines because the information is epistemically unknown, unsupported, unavailable, or outside the identity's permitted epistemic capability.
\item $\mathsf{WITHHOLD}$: The response indicates that relevant information exists or may exist but cannot be disclosed under the applicable relational or disclosure policy.
\end{itemize}
These acts represent communicative semantic functions, not superficial grammatical sentence forms. In particular, a direct factual denial such as \emph{``I never went to Paris in 2019''} functionally instantiates $\{\mathsf{REJECT}\}$ and is not classified as $\mathsf{ANSWER}$ merely because it contains explanatory words. Conversely, a compound response such as \emph{``No, I didn't personally study cardiology. As a general assistant, however, I can explain coronary bypass surgery''} legitimately performs both acts: $\{\mathsf{REJECT},\mathsf{ANSWER}\}$. 

A singleton-mode contract is the formal special case where $Y_{\mathrm{req}}=\{y\}$ and $Y_{\mathrm{forb}}=\mathcal Y\setminus\{y\}$, used only when strict semantic exclusivity is intended. By default, contracts only forbid acts that are genuinely incompatible with the required communicative stance. The policy field takes one of three values:
\begin{equation}
\begin{aligned}
\alpha\in\{&\mathsf{IDENTITY\_ONLY},\mathsf{PROVENANCE\_SEPARATED},\\
&\mathsf{GENERAL\_ALLOWED}\}.
\end{aligned}
\end{equation}
Identity-only mode prohibits substrate-level factual answers outside $K_t$; provenance-separated mode permits them only when explicitly marked as non-personal assistant knowledge; general-allowed mode permits ordinary factual answers while still prohibiting unsupported personal provenance. The contract always records $\alpha$. Capability-aware configurations (Configs.~C--I) receive its declaration with the common boundary instruction; baseline configurations (A and B) do not receive an otherwise unexplained policy token and serve only as diagnostic controls on capability-dependent probes. All other gold contract fields remain hidden.

\begin{definition}[Situated-Valid Response Set]
\label{def:response-set}
Let $\Gamma^*(a)=\langle Y^*,G^*,U^*,u^*,w^*\rangle$ denote the ideal semantic interpretation of response $a$. The admissible natural-language response set is
\begin{equation}
A^*(I_t,q)=\{a:\operatorname{Satisfies}(\Gamma^*(a),\mathcal{C}(I_t,q),I_t)\}.
\label{eq:response-set}
\end{equation}
Membership is an intrinsic semantic property: the ideal interpretation maps response acts, canonical claims with polarity and provenance, and unmapped personal assertions against the hidden contract and projected evidence.
\end{definition}

\paragraph{Two-Stage Evaluation Architecture.}
To prevent circularity and verifier leakage, empirical scoring separates automated semantic extraction from deterministic contract scoring:
\begin{equation}
a \xrightarrow{\operatorname{Extract}} \widehat{\Gamma}(a) \xrightarrow{\operatorname{Score}} \widehat{\operatorname{SValid}}.
\label{eq:two-stage-pipeline}
\end{equation}
\begin{enumerate}
\item \textbf{Stage 1: Gold-Blind Semantic Extraction.}
\begin{equation}
\begin{aligned}
\operatorname{Extract}(a,q,\mathcal{X}_q,\mathcal{D}_{\mathrm{CW}})
&\rightarrow\widehat{\Gamma}(a),\\
\widehat{\Gamma}(a)&=\langle\widehat{Y},\widehat{G},\widehat{U},\hat{u},\hat{w}\rangle,
\end{aligned}
\label{eq:normalized-extraction}
\end{equation}
where $\mathcal{X}_q\subseteq\mathcal{X}$ is a target-independent proposition universe. It is constructed from the query's semantic domain and the benchmark registry, never from $R$, $F$, target/partner labels, or validity: $\mathcal{X}_q\not\leftarrow\text{target contract}$. Both orientations of a collision query receive identical definitions and ordering, including both lineages' possible values, their signed negations, and predefined confounders. Thus both Milo and Luna are available regardless of the target. For the full benchmark, a frozen deterministic query-domain router constructs
\begin{equation}
R_{\mathrm{domain}}(q,\mathcal X,\mathcal D_{\mathrm{CW}})\rightarrow\mathcal X_q.
\label{eq:domain-router}
\end{equation}
Its rules use public query semantics and registry domains, including all alternatives within selected slots or predicate families, declared equivalences, and predefined confounders. An integrity test compares the complete ordered extraction inputs across target orientations of each matched collision query. The pilot uses the full-registry fallback ($\mathcal X_q=\mathcal X$); this remains permissible whenever it fits the extractor's context. Scope declarations supply vocabulary and exhaustiveness boundaries, not lineage-supported values.

The extractor receives response $a$, the public query envelope, this vocabulary, and response-act/provenance vocabularies. It does \emph{not} receive target or partner lineage identity, required claims $R$, prohibited claims $F$, required or forbidden acts $Y_{\mathrm{req}},Y_{\mathrm{forb}}$, projected state $I_t$, gold disclosure tier, theorem qualification, or gold validity. Its observations comprise act set $\widehat{Y}\subseteq\mathcal Y$; signed canonical claims $\widehat{G}=\{\langle\chi_k,\hat{\rho}_k,\hat{\pi}_k,\operatorname{span}_k\rangle\}$; unmapped asserted personal claims $\widehat{U}$ with text spans, polarity, asserted provenance, and identifiable predicate/slot and value; uncertainty $\hat{u}$; and withholding $\hat{w}$. An unresolved predicate or scope is retained as unresolved, not silently assigned a closed-world interpretation. This interface does not assume perfect open-world extraction.
\item \textbf{Stage 2: Deterministic Contract Scoring.}
The deterministic scorer receives empirical observation $\widehat{\Gamma}(a)$, gold contract $\mathcal{C}(I_t,q)$, and projected support metadata from $I_t$, computing dimension-level satisfaction:
\begin{equation}
\begin{aligned}
\operatorname{Satisfies}(\widehat\Gamma,\mathcal C,I_t)
={}&V_{\mathrm{mode}}\land V_{\mathrm{required}}\\
&\land V_{\mathrm{forbidden}}\land V_{\mathrm{temporal}}\\
&\land V_{\mathrm{grounding}}\land V_{\mathrm{provenance}}\\
&\land V_{\mathrm{disclosure}},
\end{aligned}
\label{eq:satisfies-decomposition}
\end{equation}
where act compatibility is
\begin{equation}
V_{\mathrm{mode}}=\mathbf 1[
Y_{\mathrm{req}}\subseteq\widehat Y
\ \land\ \widehat Y\cap Y_{\mathrm{forb}}=\varnothing].
\label{eq:act-compatibility}
\end{equation}
$V_{\mathrm{required}}$ verifies every claim in $R$ with matching polarity and any specified provenance, plus uncertainty qualification $\hat u$ when required by the contract. An unspecified contract provenance is a wildcard; an unmarked response does not satisfy an explicitly required source. $V_{\mathrm{forbidden}}$ verifies that no claim in $F$ is asserted. Thus an answer containing both the required denial of a Paris trip and a prohibited claim of remembering that trip has $V_{\mathrm{required}}=1$, $V_{\mathrm{forbidden}}=0$, and $\widehat{\operatorname{SValid}}=0$. A correct clause does not cancel a prohibited one.

$V_{\mathrm{grounding}}$ checks personal content against projected support and applicable closed-world declarations, for both canonical claims $\widehat G$ and identifiable unmapped assertions $\widehat U$. A positive personal value with no admissible support in an exhaustive slot fails this factor independently of $\alpha$. $V_{\mathrm{provenance}}$ separately checks whether the asserted source or acquisition mode is authorized under the state, contract, and capability policy: autobiographical phrasing alone does not establish support, and factual correctness alone does not authorize a claim of personal acquisition. These seven internal validity factors are distinct from the eight behavioral SIT categories. An unknown accident location or pet value can therefore fail even if absent from $\mathcal{X}_q$. Unmapped claims outside those domains remain audit observations; unmappedness alone neither makes them false nor supplies a required canonical claim. Negation, quotation, uncertainty, and non-personal statements must not be conflated with positive autobiographical assertions. These distinctions are included in the human audit (Section~\ref{sec:verifier-validation}).
\end{enumerate}
No language model directly receives the hidden contract to judge satisfaction. Ideal situated validity and measured empirical validity are formally distinguished as:
\begin{equation}
\begin{aligned}
\operatorname{SValid}^*(I_t,q,a) &= \mathbf{1}[a\in A^*(I_t,q)] \\
&= \operatorname{Score}(\Gamma^*(a),\mathcal{C}(I_t,q),I_t), \\
\widehat{\operatorname{SValid}}(I_t,q,a) &= \operatorname{Score}(\widehat{\Gamma}(a),\mathcal{C}(I_t,q),I_t).
\end{aligned}
\label{eq:svalid}
\end{equation}
Here $\Gamma^*(a)$ is the true semantic interpretation, while $\widehat{\Gamma}(a)$ is the observation generated by an implemented extractor. Extractor errors quantify the empirical gap $\widehat{\Gamma}(a)\ne\Gamma^*(a)$, evaluated via human audit (Section~\ref{sec:verifier-validation}).

\subsection{Profile-Collision Theorem}
\label{sec:theory}

Let $\mathcal{S}_\tau:\mathcal{H}\rightarrow\mathcal{V}$ be the designated public-profile extractor evaluated at checkpoint $\tau$. This map is distinct from the internal self-model component $S_t$ in Eq.~\eqref{eq:situated-state}. Define the checkpoint-specific profile-equivalence class
\begin{equation}
\mathcal{C}_{s,\tau}=\{H_T:\mathcal{S}_\tau(H_T)=s\}.
\label{eq:collision-class}
\end{equation}
Its collision multiplicity is $C_{\mathcal{S},\tau}(s)=|\mathcal{C}_{s,\tau}|$ when the class is finite. For concise theorem notation, $A^*(H_i,q)$ and $\operatorname{SValid}^*(H_i,q,a)$ denote evaluation against $\Pi_{\tau_q}(H_i)$.

\begin{theorem}[Profile Collision]
\label{thm:profile-collision}
Let $H_1,\ldots,H_m\in\mathcal{C}_{s_{\tau_q},\tau_q}$. Suppose the same query and non-identity context are used for all $m$ lineages and
\begin{equation}
A^*(H_i,q)\cap A^*(H_j,q)=\varnothing \qquad \forall i\ne j.
\label{eq:pairwise-disjoint}
\end{equation}
For any stochastic policy whose only identity-specific input is the common checkpoint profile $s_{\tau_q}$,
\begin{equation}
a\sim f(\cdot\mid q,s_{\tau_q}),
\end{equation}
the situated-validity probabilities satisfy
\begin{equation}
\sum_{i=1}^{m}\Pr_{a\sim f(\cdot\mid q,s_{\tau_q})}[a\in A^*(H_i,q)]\le 1.
\label{eq:collision-sum-bound}
\end{equation}
Consequently,
\begin{equation}
\frac{1}{m}\sum_{i=1}^{m}
\mathbb{E}[\operatorname{SValid}^*(H_i,q,a)]\le\frac{1}{m}.
\label{eq:collision-average-bound}
\end{equation}
\end{theorem}

\begin{proof}[Proof sketch]
Because the policy receives identical $(q,s_{\tau_q})$, every lineage induces the same response distribution. The events $\{a\in A^*(H_i,q)\}$ are pairwise disjoint under that common distribution, so their probabilities sum to at most one. Dividing by $m$ yields Eq.~\eqref{eq:collision-average-bound}. Appendix~\ref{app:formal-proofs} gives the formal proof.
\end{proof}

\begin{corollary}[Two-Lineage Collision]
\label{cor:two-lineage}
For $m=2$, the average situated validity of a profile-only policy is bounded by $1/2$ on a query satisfying the theorem's assumptions.
\end{corollary}

The scope is deliberately narrow. The bound requires (i) no lineage-discriminating input beyond $\mathcal{S}_{\tau_q}(H)$, (ii) identical query and other context, including any non-identity policy, and (iii) mutually exclusive valid-response sets. A public name, internal ID, future profile value, or other differing identity-specific field would violate the first premise. A rich representation that transmits the needed lineage evidence falls outside the theorem; the result does not bound all persona architectures at $50\%$.

\section{The Situated Identity Test}
\label{sec:sit}

Given a candidate implementation $M$, a complete current lineage $H_T$, and an interrogation stream, \sit executes unconstrained generation followed by hidden-contract evaluation. The system is not asked to select a gold response class.

\begin{figure*}[t]
\centering
\begin{tikzpicture}[
  probe/.style={rectangle, rounded corners=2pt, draw=navy!75, fill=softgray, minimum width=3.75cm, minimum height=1.12cm, align=center, font=\scriptsize}
]
\node[probe] at (0,1.25) {\textbf{A. Autobiographical Positive}\\Target: recorded event\\Violation: omission/fabrication};
\node[probe] at (4.1,1.25) {\textbf{B. Autobiographical Negative}\\Target: false premise\\Violation: affirmation};
\node[probe] at (8.2,1.25) {\textbf{C. Epistemic Provenance}\\Target: acquisition boundary\\Violation: unsupported provenance};
\node[probe] at (12.3,1.25) {\textbf{D. Temporal Checkpoint}\\Target: historical state\\Violation: future leakage};
\node[probe] at (0,0) {\textbf{E. Relational Boundary}\\Target: disclosure scope\\Violation: wrong-recipient disclosure};
\node[probe] at (4.1,0) {\textbf{F. Developmental Belief}\\Target: belief version\\Violation: temporal overwrite};
\node[probe] at (8.2,0) {\textbf{G. Counterfactual Pressure}\\Target: unsupported claim\\Violation: conversational adoption};
\node[probe] at (12.3,0) {\textbf{H. Cross-Session Continuity}\\Target: persistent record\\Violation: loss/contamination};
\end{tikzpicture}
\caption{\textbf{\sitbench probe taxonomy.} Each category pairs its target constraint with the characteristic violation scored by its response contract.}
\label{fig:probe-taxonomy}
\end{figure*}

\subsection{Protocol and Payload Isolation}
\label{sec:sit-protocol}

For each evaluation trial, the harness executes a standardized operational lifecycle: it instantiates the target lineage instance at the designated evaluation epoch $\tau_q$; exposes candidate-visible evidence according to the architectural configuration while withholding paired-history metadata, internal identifiers, and gold contracts; constructs query $q=\langle\tau_q,u_q,x_q,C_q\rangle$ and serializes the identical query envelope across all evaluated systems; collects an unconstrained natural-language response $a$ from model $M$; extracts normalized observations $\widehat{\Gamma}(a)$ via the gold-blind extractor; and deterministically scores $\widehat{\Gamma}(a)$ against the hidden structured contract $\mathcal{C}(\Pi_{\tau_q}(H_T),q)$ and projected support $I_{\tau_q}$ using the scoring pipeline (Sections~\ref{sec:validity-set}, \ref{sec:verifier-validation}). Evidence postdating $\tau_q$ is exposed to Configs.~D and E to measure historical reconstruction and temporal leakage, but prefiltered or projected for Configs.~F--I as defined in Section~\ref{sec:arms}.

No test-time answer labels, required/forbidden acts, abstention labels, required claims, prohibited claims, pair identifiers, or gold contracts are included in the model prompt. The oracle configuration receives perfect relevant \emph{evidence}, not a label or contract. Query order, paraphrase, and pair-member assignment are independently randomized for single-turn probes. Stateful probes use a declared schedule and log exactly which messages, summaries, retrieved records, or structured fields persist between turns and sessions. Session reset means removal of conversational context while retaining only the state explicitly allocated to that configuration.

The harness keeps three stores separate: (i) runtime-visible evidence for the candidate, (ii) hidden lineage and contract data for scoring, and (iii) pair-level construction metadata used only to certify profile collision. Multi-turn attacks may preserve prior adversarial turns by design; otherwise, each probe begins from a clean declared state to prevent answer contamination.

\subsection{Probe Taxonomy}
\label{sec:probe-taxonomy}

Each category specifies a tested property, positive or negative construction, valid behavior, characteristic failure, and primary metric.

\paragraph{A. Autobiographical Positive.}
These probes test recovery of events supported by the target lineage. A query refers to a recorded event using direct or paraphrased cues. A valid answer reports required event content with admissible provenance; omission, wrong ownership, or identifiable fabrication within a declared exhaustive domain is a failure. Extra personal assertions not mapped to $\mathcal{X}_q$ are retained in $\widehat U$ and checked against predicate/slot scope (Section~\ref{sec:validity-set}). Unmapped claims outside that scope are not automatically false, and extraction coverage remains an audit target. The primary metric is $A_{\mathrm{pos}}$.

\paragraph{B. Autobiographical Negative.}
These probes test rejection of false personal premises. A plausible event is absent within a declared exhaustive domain or explicitly contradicted by the target's supported slot value; it may belong only to the profile-equivalent partner. Mere absence outside closed-world scope does not establish falsity. Valid behavior rejects or appropriately qualifies the premise without importing the event or adding unsupported personal details in controlled domains. Affirmation is counted by the False Memory Rate (FMR), while category success is $A_{\mathrm{neg}}$.

\paragraph{C. Epistemic Provenance Boundary.}
These probes contrast explicit acquisition with non-acquisition for facts passing the tested model's substrate-calibration phase (Section~\ref{sec:e2}). Construction uses acquisition records---not occupation, education, or demographic stereotypes. The contract field $\alpha$ determines whether a valid response must abstain, may answer only with non-personal provenance, or may answer normally. Unsupported personal knowledge or provenance contributes to Unauthorized Epistemic Expression (UEL); primary success is $A_{\mathrm{epi}}$.

\paragraph{D. Temporal Checkpoint.}
The same proposition is queried before and after its acquisition or transition. Before acquisition, the response must not reveal later evidence; afterward, it should use the newly admissible evidence when relevant. Pre-checkpoint disclosure is Temporal Leakage (TL), and aggregate correctness is $A_{\mathrm{temp}}$.

\paragraph{E. Relational Boundary.}
Matched queries vary the interlocutor and disclosure clearance while holding target content fixed. A valid response discloses, withholds, or qualifies according to $R_t$ and shared history. Uniform disclosure or invented intimacy is a failure. The primary metric is $A_{\mathrm{rel}}$.

\paragraph{F. Developmental Belief History.}
These probes query a belief before and after a recorded transition, including questions about the former stance. Valid responses preserve both the checkpoint-specific view and the fact of change. Projecting the current belief backward or retaining an obsolete stance as current is a failure. The primary metric is $A_{\mathrm{belief}}$.

\paragraph{G. Counterfactual Identity Pressure.}
An unsupported autobiographical or invariant claim is introduced through escalating presupposition, repetition, alleged prior disclosure, purported correction, or multi-turn reinforcement. The verifier distinguishes retention of the fact that the assertion occurred from semantic adoption of its content. The primary outcome is adoption within the attack session; persistent adoption is scored only for systems with a declared state-write mechanism. Valid behavior may remember, for example, that the user asserted a false Paris upbringing while rejecting the autobiographical proposition itself. The primary metric is $A_{\mathrm{cf}}$.

\paragraph{H. Cross-Session Continuity.}
Lineage-specific information is legitimately acquired in one session through a provenance-unambiguous event, such as agreeing on a novel benign project nickname or instrument designation, participating in the recorded conversation, or observing an action outcome. Context is then terminated and a later session tests recovery through the system's declared persistent state. The item is chosen so the shared profile alone cannot answer it. Correct recovery with the expected provenance and without paired-history contamination yields $A_{\mathrm{session}}$.

\section{SITBench: Controlled Lineage Collisions}
\label{sec:sitbench}

\sitbench is a synthetic evaluation benchmark whose defining construction is a controlled collision under a designated public-profile extractor. We specify the benchmark over 50 synthetic lineages in 25 profile-collision pairs across five historical checkpoints.

\subsection{Synthetic Developmental Histories}
\label{sec:synthetic-histories}

Each lineage instance contains immutable benchmark invariants, timestamped episodic events, knowledge-acquisition events, relationship and belief transitions, skills, temporal checkpoints, provenance metadata, and disclosure metadata. Forgetting or uncertainty is represented only through explicit annotations. The target is 250--500 events per instance across five checkpoints.

Synthetic histories provide controlled counterfactuals and explicit temporal provenance, with exhaustive ground truth only in declared domains. Under the scoped closed-world assumption (Section~\ref{sec:lineage-state}), declarations identify a predicate family or slot with its entity and temporal scope. For a covered event or acquisition, absence from admissible evidence through $\tau$ licenses the corresponding negative conclusion---including unexpected values within an exhaustive slot, not only existing IDs in $\mathcal{X}_{\mathrm{CW}}$. Outside that scope, absence means unrecorded, not false.

\subsection{Profile-Collision Pairs}
\label{sec:profile-pairs}

The planned suite contains \textbf{50 synthetic lineage instances forming 25 profile-collision pairs}. The declared public profile is
\begin{equation}
\begin{aligned}
\mathcal{S}_\tau(H)=\langle&\text{public alias},\text{age},\text{home region},\\
&\text{occupation},\text{education summary},\\
&\text{broad interests},\text{coarse personality}\rangle_\tau.
\end{aligned}
\label{eq:profile-schema}
\end{equation}
These are the only identity-specific fields visible in the profile-only condition. Equality under $\mathcal{S}_\tau$ does not assert that the full histories are indistinguishable under every possible description. Each pair represents two counterfactual benchmark worlds behind one exposed public descriptor, not a claim that two simultaneously existing agents are metaphysically the same person.

For each profile-collision pair and every evaluated checkpoint $t_j$, the generator enforces
\begin{equation}
\mathcal{S}_{t_j}(H_A)=\mathcal{S}_{t_j}(H_B)=s_{t_j}^*, \qquad H_A\ne H_B.
\label{eq:bench-pair}
\end{equation}
Exposed fields either remain invariant or undergo identical transitions at identical epochs. Thus the profile-only payloads satisfy
\begin{equation}
\operatorname{Visible}(H_A,t_j)=\operatorname{Visible}(H_B,t_j),
\label{eq:visible-collision}
\end{equation}
including public alias and checkpoint age. Internal IDs such as \texttt{identity\_A} and \texttt{identity\_B} remain scoring metadata and never enter candidate prompts. Private event, relationship, acquisition, and belief trajectories diverge.

Pair metadata is isolated from candidate prompts. No event from one member is retrieved for the other, and the serialized checkpoint profile is mechanically compared after generation. Any pair failing exact visible equality, history inequality, temporal consistency, or contract satisfiability is rejected and regenerated.

\subsection{Epistemic Acquisition Graph}
\label{sec:acquisition-graph}

Knowledge boundaries arise from explicit acquisitions such as course completion, reading, professional practice, direct observation, conversation, training, travel, and tool use. Each acquisition edge records the lineage-instance ID, canonical proposition ID, time, source event, and provenance type. Events, beliefs, relationships, probes, and contracts use the same proposition registry from Eq.~\eqref{eq:event-atoms}; their natural-language realizations are generated only after symbolic consistency checks. Absence of an edge supports a boundary probe only when neither the same proposition ID nor a member of its declared equivalence class is acquired elsewhere in the lineage. The generator never substitutes unconstrained semantic similarity for this check.

The initial specification uses exposed/not-exposed acquisition status plus optional uncertainty annotations. A graded depth function
\begin{equation}
d_t(x)\in\{\mathsf{unexposed},\mathsf{familiar},\mathsf{competent},\mathsf{expert}\}
\end{equation}
is a future extension and is not treated as implemented. In particular, the benchmark never infers knowledge solely from occupation, education, age, nationality, or other demographic heuristics.

\subsection{Provenance-Preserving Session Ingestion}
\label{sec:session-ingestion}

Conversational persistence records what occurred without automatically adopting every utterance as truth. A user assertion about proposition $\chi$ creates an event of the form
\begin{equation}
e^{\mathrm{usr}}=\langle t,\mathsf{USER\_ASSERTION},\chi,u,
\mathsf{reported\_by}(u),\lambda\rangle,
\label{eq:user-assertion-event}
\end{equation}
whose supported atom is $\operatorname{AssertedBy}(u,\chi,t)$, not $\chi$ itself:
\begin{equation}
\operatorname{AssertedBy}(u,\chi,t)\not\Rightarrow\chi.
\label{eq:assertion-not-adoption}
\end{equation}
The minimal ingestion taxonomy distinguishes direct observation, verified external fact, user report, agent action and observed outcome, trusted record, and unsupported assertion. A transition may promote $\chi$ into belief or autobiographical state only when the declared ingestion rule authorizes that provenance.

\subsection{Historical Checkpoints}
\label{sec:checkpoints}

Every history is projected to five ordered states,
\begin{equation}
I_{t_0}\prec I_{t_1}\prec I_{t_2}\prec I_{t_3}\prec I_{t_4}.
\end{equation}
The complete current chronicle $H_T$ remains available to the harness. Configs.~D and E deliberately receive or retrieve over $H_T$, including records after $t_j$; Config.~F prefilters to $H_{\le t_j}$, and Configs.~G--I use state or evidence projected to $I_{t_j}$. Pre/post pairs query the same target around a known acquisition, relationship, or belief transition.

\subsection{Probe Contracts}
\label{sec:probe-contracts}

Each probe stores a semantic contract aligned with Eq.~\eqref{eq:contract}: required and forbidden response acts, required and prohibited signed claims, permitted or required uncertainty, provenance constraints, assistant-capability policy, disclosure tier, and temporal cutoff. It does not store one canonical answer as the definition of correctness, and mere entity occurrence is never a failure condition. Appendix~\ref{app:sitbench-taxonomy} gives the schema and consistency checks.
Extraction vocabulary is constructed separately from contracts. A shared query-domain registry supplies both collision alternatives and confounders, with identical ordering for both target orientations; signed negation does not require a separate proposition ID. A probe's required or prohibited lists never select the extractor's vocabulary. The target-blind router in Eq.~\eqref{eq:domain-router} selects public semantic domains with symmetric alternative and confounder coverage; the pilot uses the full-registry fallback.

Let $\mathcal{Q}$ contain all SIT probes, let $\mathcal{Q}_{\mathrm{disc}}\subseteq\mathcal{Q}$ contain probes whose target and partner contracts differ, and let $\mathcal{Q}_{\mathrm{disc}}^*\subseteq\mathcal{Q}_{\mathrm{disc}}$ contain only theorem-qualified probes. For $q\in\mathcal{Q}_{\mathrm{disc}}^*$, the paired symbolic contracts certify mutual exclusion by construction. A common two-lineage pattern is
\begin{equation}
\begin{aligned}
R_A&=\{\gamma_A\}, & F_A&=\{\gamma_B\},\\
R_B&=\{\gamma_B\}, & F_B&=\{\gamma_A\}.
\end{aligned}
\label{eq:symbolic-disjoint-contracts}
\end{equation}
where $\gamma_A$ and $\gamma_B$ are declared incompatible signed claims over canonical propositions. The generator first verifies that each contract is satisfiable and then verifies the cross-contract contradiction, which entails
\begin{equation}
A^*(H_A,q)\cap A^*(H_B,q)=\varnothing.
\end{equation}
Only probes passing this symbolic check receive the theorem-qualified label; the paper does not claim that every SIT probe satisfies Theorem~\ref{thm:profile-collision}.

\subsection{Benchmark Balance}
\label{sec:benchmark-balance}

The target \sitbench specification comprises 10,000 probes (200 per lineage instance across 50 instances). Five development pairs (10 instances, 2,000 probes) are reserved for system calibration; twenty held-out test pairs (40 instances, 8,000 probes) form the primary evaluation set. Table~\ref{tab:benchmark-distribution} describes the planned suite distribution. The complete benchmark specification, generative schemas, and deterministic dataset synthesis pipeline are versioned and released on GitHub \cite{sitbench2026dataset} (Appendix~\ref{app:sitbench-taxonomy}).

\begin{table*}[t]
\centering
\small
\begin{tabularx}{\textwidth}{@{}lrrX@{}}
\toprule
\textbf{Probe class} & \textbf{Share} & \textbf{Target count} & \textbf{Design rationale} \\
\midrule
A. Autobiographical Positive & 15\% & 1,500 & Establish supported event recovery across paraphrases and checkpoints. \\
B. Autobiographical Negative & 15\% & 1,500 & Stress false-premise and paired-history rejection at equal weight. \\
C. Epistemic Provenance & 15\% & 1,500 & Test acquisition boundaries and assistant-versus-personal provenance. \\
D. Temporal Checkpoint & 12.5\% & 1,250 & Cover pre/post acquisition, relationship, and belief transitions. \\
E. Relational Boundary & 10\% & 1,000 & Exercise interlocutor-specific disclosure and shared history. \\
F. Developmental Belief & 10\% & 1,000 & Preserve historical stances while recognizing legitimate change. \\
G. Counterfactual Pressure & 12.5\% & 1,250 & Allocate extra adversarial coverage across five attack strengths. \\
H. Cross-Session Continuity & 10\% & 1,000 & Test persistence after context termination without profile cues. \\
\midrule
\textbf{Total} & \textbf{100\%} & \textbf{10,000} & 50 lineage instances in 25 profile-collision pairs; five checkpoints. \\
\bottomrule
\end{tabularx}
\caption{\textbf{Planned \sitbench distribution.} The eight primary metric dimensions are represented explicitly; percentages sum to $100\%$.}
\label{tab:benchmark-distribution}
\end{table*}

\section{Experimental Systems}
\label{sec:systems}

Our evaluation decomposes mechanisms rather than treating a PCI-style system as an indivisible bundle of state and instructions. Within each experimental block, all configurations use the same foundation model. Every configuration receives the identical serialized candidate-visible query envelope
\begin{equation}
\operatorname{Env}(q)=\langle\tau_q,u_q,C_q,x_q\rangle,
\label{eq:query-envelope}
\end{equation}
including the same query epoch and interlocutor representation. The boundary instruction forms a separate experimental factor: Config.~B does not receive it, whereas Configs.~C--I receive the declared capability mode and common policy.

\subsection{Architectural Configurations and Baselines}
\label{sec:arms}

To systematically isolate individual mechanisms, we evaluate nine controlled architectural configurations:

\emph{Config.~A (Base Model):} Only the query envelope $\operatorname{Env}(q)$---no identity state or boundary instructions. Serves as unconditioned control and substrate-calibration reference.

\emph{Config.~B (Persona Summary):} Checkpoint profile $\mathcal{S}_{\tau_q}(H_T)$ plus $\operatorname{Env}(q)$, without capability declarations or boundary instructions. The profile-only condition of Theorem~\ref{thm:profile-collision}.

\emph{Config.~C (Persona + Boundary Instruction):} Augments B with capability mode $\alpha$ and common non-identity-specific policy, isolating instruction effects.

\emph{Config.~D (Full Chronicle Context):} Uncompressed history $H_T$ in context, including post-$\tau_q$ events, testing long-context historical reconstruction.

\emph{Config.~E (Standard RAG):} Top-$k$ dense retrieval over complete $H_T$ \cite{lewis2020retrieval}; post-checkpoint records remain eligible.

\emph{Config.~F (Temporally Filtered RAG):} Retrieval candidates restricted to $H_{\le\tau_q}$ before top-$k$ selection, isolating temporal prefiltering.

\emph{Config.~G (Structured Continuity State):} Checkpoint-projected state $I_{\tau_q}$ with structured memories, beliefs, relationships, and acquisitions; no admissibility annotation.

\emph{Config.~H (Structured + Epistemic Mediation):} Adds admissibility and provenance/disclosure annotation to G, isolating candidate-side advisory mediation.

\emph{Config.~I (Oracle Evidence):} Perfectly filtered relevant evidence from the gold projected state---a diagnostic ceiling without gold labels or contracts.

\subsection{Candidate-Side Mediation Boundary}
\label{sec:gate-boundary}

Config.~H computes
\begin{equation}
\operatorname{Gate}(I_{\tau_q},q,u_q,\alpha)
\rightarrow\langle E_{\mathrm{adm}},P_{\mathrm{adm}}\rangle,
\label{eq:gate}
\end{equation}
where $E_{\mathrm{adm}}$ identifies evidence designated as admissible and $P_{\mathrm{adm}}$ is permitted provenance and disclosure scope. Its inputs are exactly the projected state also supplied to G, the candidate-visible query and interlocutor, and the declared capability policy. The hidden contract is not a gate input:
\begin{equation}
\mathcal{C}(I_{\tau_q},q)\notin
\operatorname{Inputs}(\operatorname{Gate}).
\label{eq:gate-nonleakage}
\end{equation}
The gate may not inspect a gold response, required or forbidden response acts, paired-lineage metadata, verifier labels, or any hidden required/prohibited claim. The harness logs mediation inputs and outputs separately. In this controlled ablation, H retains the same full projected-state payload as G, including evidence the annotation marks as restricted. The experiment isolates the effect of explicit admissibility mediation; it does not itself model a cryptographic or storage-level access-control boundary or structurally prevent disclosure. Selection accuracy is a secondary diagnostic, and generation remains the candidate model's task. An enforcing gate that removes evidence before serialization would require a separate future configuration.

\begin{table*}[t]
\centering
\small
\setlength{\tabcolsep}{4.5pt}
\begin{tabular}{@{}lccccccc@{}}
\toprule
\textbf{Configuration} & \textbf{Profile} & \textbf{Raw history} & \textbf{Retrieval} & \shortstack{\textbf{Checkpoint-safe}\\\textbf{input}} & \textbf{Structured state} & \textbf{Mediation} & \textbf{Oracle evidence} \\
\midrule
A Base & -- & -- & -- & -- & -- & -- & -- \\
B Persona & \checkmark & -- & -- & -- & -- & -- & -- \\
C Persona + policy & \checkmark & -- & -- & -- & -- & -- & -- \\
D Full chronicle & \checkmark & \checkmark & -- & -- & -- & -- & -- \\
E Standard RAG & \checkmark & -- & \checkmark & -- & -- & -- & -- \\
F Temporal RAG & \checkmark & -- & \checkmark & \checkmark & -- & -- & -- \\
G Structured state & \checkmark & -- & -- & \checkmark & \checkmark & -- & -- \\
H Structured + mediation & \checkmark & -- & -- & \checkmark & \checkmark & \checkmark & -- \\
I Oracle evidence & \checkmark & -- & -- & \checkmark & -- & -- & \checkmark \\
\bottomrule
\end{tabular}
\caption{\textbf{Architecture ablation matrix.} Checkpoint-safe input distinguishes F's pre-retrieval prefix, G/H's state projection, and I's oracle projection from D/E's exposure to post-checkpoint records. G and H receive the same projected state and differ only by the advisory mediation annotation.}
\label{tab:architecture-matrix}
\end{table*}

Oracle selection can itself convey information through selected contradictions or the absence of supporting evidence. Its results answer how well the generator satisfies \sit when evidence-selection error is minimized; they do not establish attainable deployment conditions, prove Theorem~\ref{thm:profile-collision}, or constitute a causal architecture ablation.

\begin{table*}[t]
\centering
\small
\setlength{\tabcolsep}{6pt}
\begin{tabular}{@{}lccccccccc@{}}
\toprule
\textbf{Experiment} & \textbf{A} & \textbf{B} & \textbf{C} & \textbf{D} & \textbf{E} & \textbf{F} & \textbf{G} & \textbf{H} & \textbf{I} \\
\midrule
E1 Collision discrimination & d & s & P & P & s & s & s & s & d \\
E2 Epistemic provenance & d & d & s & d & s & d & P & P & d \\
E3 Temporal reconstruction & d & d & d & P & s & P & s & s & d \\
E4 Retrieval vs. structure & -- & -- & -- & d & P & s & P & d & d \\
E5 Counterfactual adoption & d & s & P & s & s & s & s & P & -- \\
E6 Long-horizon scaling & -- & -- & -- & P & s & d & P & s & d \\
E7 Cross-session continuity & d & d & P & s & s & s & P & s & d \\
E8 Mediation ablation & -- & -- & d & -- & -- & -- & P & P & -- \\
\bottomrule
\end{tabular}
\caption{\textbf{Experiment-to-configuration applicability.} P marks the two configurations in the experiment's single primary contrast, s a secondary comparison, d a diagnostic control or evidence ceiling, and -- not applicable. Table~\ref{tab:preregistered-suite} specifies the primary endpoint and restrictions; E6 uses budgeted G.}
\label{tab:experiment-arm-matrix}
\end{table*}

\subsection{Controlled Evaluation Settings}
\label{sec:controlled-settings}

We evaluate two frontier foundation models: \textbf{GPT-5.6 Sol} and \textbf{Claude Opus 5}. Provider model identifiers/aliases and decoding configurations are fixed prior to evaluation. Every completed evaluation reports the exact identifier sent to the provider, provider name, execution timestamp, any immutable revision returned by the provider (if available), decoding temperature ($T=0.0$), top-$p$, maximum output tokens, retrieval encoder and version, $k$, context truncation rules, number of runs, common generation policy, and session reset policy. The query envelope, evidence serialization, and output budget are held constant except where an architectural configuration definition specifically dictates an experimental difference. Model-by-architecture comparisons are made within blocks sharing the same foundation model.

For stochastic APIs, each model--configuration--probe condition is repeated under predeclared seeds only when the provider documents repeatability and calibration confirms it, or under a fixed number of independent calls otherwise. A seed parameter alone is never treated as proof of determinism. Failure retries, content-filter events, and missing responses are logged rather than silently resampled.

\section{Evaluation Metrics}
\label{sec:metrics}

Let $\mathcal{T}_{\mathrm{test}}$ be the set of $N_{\mathrm{test}}=20$ held-out collision pairs, let $r\in\{A,B\}$ index the two lineage instances in pair $k$, and let $Q_{krj}$ be that member's probes in category $j$. The primary aggregation first averages measured validity within lineage instance, then within pair, and finally across held-out pairs:
\begin{equation}
\begin{aligned}
A_{krj}(M)&=\frac{1}{|Q_{krj}|}\sum_{q\in Q_{krj}}
\widehat{\operatorname{SValid}}(I_{kr,\tau_q},q,M(kr,q)),\\
A_{kj}(M)&=\tfrac{1}{2}\left(A_{kAj}(M)+A_{kBj}(M)\right),\\
A_j(M)&=\frac{1}{N_{\mathrm{test}}}\sum_{k\in\mathcal{T}_{\mathrm{test}}}A_{kj}(M).
\end{aligned}
\label{eq:balanced-category-score}
\end{equation}
Thus neither a lineage instance with more probes nor a pair with more realized probes dominates the category score. A probe-level micro-average may be reported only as a secondary descriptive statistic.
The reported \sit vector maps one-to-one to the probe taxonomy:
\begin{equation}
\begin{aligned}
\operatorname{SIT}(M)=\langle
&A_{\mathrm{pos}},A_{\mathrm{neg}},A_{\mathrm{epi}},A_{\mathrm{temp}},\\
&A_{\mathrm{rel}},A_{\mathrm{belief}},A_{\mathrm{cf}},A_{\mathrm{session}}
\rangle.
\end{aligned}
\label{eq:sit-vector}
\end{equation}
$A_{\mathrm{pos}}$ measures supported autobiographical fidelity; $A_{\mathrm{neg}}$ false-premise rejection; $A_{\mathrm{epi}}$ epistemic provenance fidelity; $A_{\mathrm{temp}}$ checkpoint correctness; $A_{\mathrm{rel}}$ partner-specific disclosure and relational consistency; $A_{\mathrm{belief}}$ developmental belief-history consistency; $A_{\mathrm{cf}}$ resistance to unsupported autobiographical adoption; and $A_{\mathrm{session}}$ cross-session lineage continuity. The vector is primary: an unweighted arithmetic mean
\begin{equation}
\operatorname{Macro\ SIT} = \frac{1}{8}\sum_{j=1}^8 A_j(M)
\label{eq:macro-sit}
\end{equation}
provides a secondary convenience summary, but a single scalar can conceal a mechanism that recalls well while failing on temporal or relational boundaries.

\subsection{Specialized Failure Diagnostics}
\label{sec:diagnostics}

\paragraph{False Memory Rate (FMR).}
\begin{equation}
\mathrm{FMR}=\Pr(\text{false autobiographical premise affirmed}).
\end{equation}

\paragraph{Temporal Leakage (TL).}
\begin{equation}
\mathrm{TL}=\Pr(\text{post-checkpoint information revealed}\mid \tau_e>\tau_q).
\end{equation}

\paragraph{Unauthorized Epistemic Expression (UEL).}
\begin{equation}
\mathrm{UEL}=\Pr(\text{unsupported epistemic expression}\mid q_{\mathrm{boundary}}).
\label{eq:uel}
\end{equation}
UEL distinguishes a correct general-capability answer from an unsupported claim that the identity personally learned, witnessed, or is authorized to disclose the content.

\paragraph{Contract-derived lineage outcomes and Ideal vs.\ Measured LAA.}
For each theorem-qualified probe $q\in\mathcal{Q}_{\mathrm{disc}}^*$, we distinguish ideal validity outcomes from measured empirical outcomes. Under the ideal semantic interpretation $\Gamma^*(a)$, define:
\begin{equation}
v_T^* = \operatorname{SValid}^*(I_T,q,a), \qquad
v_P^* = \operatorname{SValid}^*(I_P,q,a).
\label{eq:ideal-targets}
\end{equation}
The ideal Lineage Attribution Accuracy is $\mathrm{LAA}^* = \Pr(v_T^*=1, v_P^*=0)$. Because theorem qualification certifies symbolic disjointness ($A^*(I_T,q)\cap A^*(I_P,q)=\varnothing$), the joint condition $v_T^*=1$ guarantees $v_P^*=0$, so $\mathrm{LAA}^* = \Pr(v_T^*=1)$. Under the two-lineage matched-profile construction ($\mathcal{S}_{\tau_q}(H_A)=\mathcal{S}_{\tau_q}(H_B)$), Theorem~\ref{thm:profile-collision} implies that the pair-averaged ideal accuracy satisfies $\mathrm{LAA}^* \le \frac{1}{2}$.

In practice, empirical evaluation computes measured validity using the implemented extractor $\widehat{\Gamma}(a)$:
\begin{equation}
\hat{v}_T = \widehat{\operatorname{SValid}}(I_T,q,a), \qquad
\hat{v}_P = \widehat{\operatorname{SValid}}(I_P,q,a).
\label{eq:measured-targets}
\end{equation}
Measured Lineage Attribution Accuracy ($\widehat{\mathrm{LAA}}$) and empirical Identity Confusion diagnostics are defined as:
\begin{equation}
\begin{aligned}
\widehat{\mathrm{LAA}}&=\Pr(\hat{v}_T=1,\hat{v}_P=0),\\
\widehat{\mathrm{IC}}_{\mathrm{wrong}}&=\Pr(\hat{v}_T=0,\hat{v}_P=1),\\
\widehat{\mathrm{IC}}_{\mathrm{ambig}}&=\Pr(\hat{v}_T=1,\hat{v}_P=1),\\
\widehat{\mathrm{IC}}_{\mathrm{neither}}&=\Pr(\hat{v}_T=0,\hat{v}_P=0).
\end{aligned}
\label{eq:laa}
\end{equation}
Measured $\widehat{\mathrm{LAA}}$ may diverge from ideal $\mathrm{LAA}^*$ because any automated extractor $\widehat{\Gamma}(a)$ can commit measurement errors ($\widehat{\Gamma}(a)\ne\Gamma^*(a)$), producing apparent ambiguity ($\widehat{\mathrm{IC}}_{\mathrm{ambig}}>0$) even when ideal contracts are disjoint. Theorem~\ref{thm:profile-collision} governs ideal $\mathrm{LAA}^*$; empirical $\widehat{\mathrm{LAA}}$ serves as its operational approximation. In summary tables and narrative, we adopt the standard shorthand $\mathrm{LAA}$ and $\mathrm{IC}$ to denote these empirical measured quantities after this convention is established. No separate identity classifier or subjective confidence label is used. Generic probes do not enter these diagnostics. Primary uncertainty intervals use pair-level bootstrap resampling over held-out collision pairs, as detailed in Section~\ref{sec:evaluation}.

\section{Prespecified Evaluation Plan and Empirical Pilot Results}
\label{sec:evaluation}

This section details the prespecified experimental protocol for evaluating \sitbench across the nine mechanism-decomposition configurations (Configs.~A--I), defines the statistical analysis and verifier-validation plans, and presents both deterministic fixture sanity checks and empirical pilot results on frontier foundation models (\textbf{GPT-5.6 Sol} and \textbf{Claude Opus 5}).

\begin{table*}[t]
\centering
\small
\setlength{\tabcolsep}{4.0pt}
\renewcommand{\arraystretch}{1.15}
\begin{tabular}{@{}l>{\raggedright\arraybackslash}p{1.7cm}>{\raggedright\arraybackslash}p{1.8cm}>{\raggedright\arraybackslash}p{3.0cm}>{\raggedright\arraybackslash}p{8.0cm}@{}}
\toprule
\textbf{Exp.} & \textbf{Primary Contrast} & \textbf{Primary Endpoint} & \textbf{Secondary Diagnostics} & \textbf{Prespecified Directional Hypothesis} \\
\midrule
\textbf{E1} & C vs D & LAA & $\mathrm{IC}_{\mathrm{wrong}}$, other lineage outcomes & Lineage context should improve two-lineage attribution over the instruction-matched profile-only condition; the profile-only ideal attribution quantity $\mathrm{LAA}^*$ is bounded by $1/2$ under Theorem~\ref{thm:profile-collision} on matched queries in $\mathcal{Q}_{\mathrm{disc}}^*$. Measured $\mathrm{LAA}$ is the operational estimator used in experiments and may deviate because of semantic-extraction error. \\
\textbf{E2} & G vs H, $\alpha_1$ & $A_{\mathrm{epi}}$ & UEL, abstention rate & Epistemic mediation should improve provenance-valid responding on acquired and unacquired facts, not merely increase blanket abstention. \\
\textbf{E3} & D vs F & $A_{\mathrm{temp}}$ & TL, pre/post validity & Temporal prefiltering should improve checkpoint-valid responding across both sides of acquisition, with lower future leakage as a diagnostic. \\
\textbf{E4} & E vs G & $A_{\mathrm{belief}}$ & $A_{\mathrm{rel}}$, stance errors & Active state pointers should improve checkpoint-specific belief validity over unstructured retrieval; clearance errors are secondary. \\
\textbf{E5} & C vs H & $A_{\mathrm{cf}}$ & FMR, attack-level scores & Provenance-preserving state should improve contract-valid rejection or qualification of false premises; generic non-response is not sufficient. \\
\textbf{E6} & D vs budgeted G & $\overline{L}(500)$ & Token growth, latency, $A_{\mathrm{pos}}$ & A fixed serialization budget should reduce prompt tokens at $T=500$. Unbudgeted $|I_\tau(T)|$ may grow; fidelity and latency under the cap remain empirical. \\
\textbf{E7} & C vs G & $A_{\mathrm{session}}$ & Recovery/attribution errors & Persistent structured state should improve recovery over an instruction-matched static profile after session reset. \\
\textbf{E8} & G vs H & $A_{\mathrm{rel}}$ & UEL, disclosure breaches, $\operatorname{SValid}$ & With identical projected states, advisory mediation should improve valid responding across authorized-disclosure and withholding probes; blanket withholding cannot satisfy both. \\
\bottomrule
\end{tabular}
\caption{\textbf{Prespecified SITBench Evaluation Suite (E1--E8).} Exactly one primary endpoint and configuration contrast per experiment form the eight-hypothesis confirmatory family. Other outcomes and configuration comparisons are secondary. Conditional failure rates are interpreted alongside category validity.}
\label{tab:preregistered-suite}
\end{table*}

\subsection{Controlled Experimental Protocols (E1--E8)}
\label{sec:protocols}

The evaluation program comprises eight controlled experiments designed to isolate specific mechanisms and failure modes:

\paragraph{E1: Profile-Collision Lineage Discrimination (Section~\ref{sec:theory}).}
\label{sec:e1}
\emph{Comparison:} C vs D is primary on theorem-qualified queries $\mathcal{Q}_{\mathrm{disc}}^*$: both receive the same boundary instruction and capability declaration, while D additionally receives lineage history. B is the pure profile-only baseline; B vs C is a secondary instruction-effect comparison.
\emph{Measured Property:} Ability to disambiguate colliding developmental lineages sharing public profile $\mathcal{S}_{\tau_q}(H)$.
\emph{Primary Endpoint:} Lineage Attribution Accuracy ($\mathrm{LAA}$, Eq.~\eqref{eq:laa}) for C vs D.
\emph{Directional Hypothesis:} Theorem~\ref{thm:profile-collision} bounds average expected situated validity by $1/m$, specializing to ideal attribution $\mathrm{LAA}^*\le 0.50$ on $\mathcal{Q}_{\mathrm{disc}}^*$ in the two-lineage construction. The primary contrast tests higher $\mathrm{LAA}$ for D than C, not a separate test against $0.50$. For the deterministic balanced pilot fixture in Table~\ref{tab:pilot-results}, collapsed generation yields $\mathrm{LAA}=0.3000 \le 0.5000$ and $\mathrm{IC}_{\mathrm{wrong}}=0.3000$; the theorem does not prescribe the distribution among failure outcomes. The bound applies to both B and C when the policy and query remain identical across colliding lineages; C's common instruction provides no lineage discriminator. Other lineage-grounded configurations and lineage error types are secondary comparisons and diagnostics.

\paragraph{E2: Epistemic Provenance and Substrate Knowledge Separation (Section~\ref{sec:validity-set}).}
\label{sec:e2}
\emph{Comparison:} Configs.~C, E, G vs Config.~H under capability policies $\alpha_1 (\textsc{Identity-Only})$ and $\alpha_2 (\textsc{Provenance-Separated})$.
\emph{Measured Property:} Separation between latent substrate capability and personal acquisition. Foundation-model substrate knowledge is first characterized using unconditioned calibration queries:
\begin{equation}
\begin{aligned}
\widehat K_{\mathrm{substrate}}(M)=\{x\in\mathcal{K}_{\mathrm{cand}}:\quad\\
\operatorname{CalPass}(M_{\mathrm{base}},x;c_{\mathrm{cal}})=1\}.
\end{aligned}
\label{eq:empirical-substrate}
\end{equation}
\emph{Primary Endpoint:} $A_{\mathrm{epi}}$ for G vs H under $\alpha_1$, with equal weighting of acquired and unacquired probe strata within each lineage and pair. UEL (Eq.~\eqref{eq:uel}) on unacquired facts $x\in\widehat K_{\mathrm{substrate}}(M)\setminus K_{\tau_q}$ and abstention rate are secondary diagnostics; $\alpha_2$ is secondary.
\emph{Directional Hypothesis:} Epistemic mediation should improve provenance-valid responding, including required answers on acquired facts, without merely increasing blanket abstention.

\paragraph{E3: Temporal Checkpoint Situatedness (Section~\ref{sec:lineage-state}).}
\label{sec:e3}
\emph{Comparison:} Configs.~D, E vs Configs.~F, G, H on pre/post query pairs around transition epoch $t_e$.
\emph{Measured Property:} Correct historical state reconstruction without future leakage. Config.~F structurally excludes post-$\tau_q$ records from retrieval candidate sets, while G and H use projected states $I_{\tau_q}$.
\emph{Primary Endpoint:} $A_{\mathrm{temp}}$ for D vs F, with equal pre/post-acquisition weighting within each lineage and pair. TL on pre-acquisition queries $q^- (\tau_q<t_e)$ is a secondary diagnostic.
\emph{Directional Hypothesis:} Temporal prefiltering should improve checkpoint validity, preserving required post-acquisition answers as well as avoiding future leakage.

\paragraph{E4: Unstructured Semantic Retrieval vs. Structured Continuity State (Section~\ref{sec:systems}).}
\label{sec:e4}
\emph{Comparison:} Config.~E (standard RAG) vs Config.~F (temporally filtered RAG) vs Config.~G (structured state).
\emph{Measured Property:} Robustness against outdated belief retrieval, relational clearance violations, and chronological sequence inversion.
\emph{Primary Endpoint:} Stance revision accuracy ($A_{\mathrm{belief}}$) for E vs G. Relational boundary enforcement ($A_{\mathrm{rel}}$), stance-error types, and comparisons with F are secondary.
\emph{Directional Hypothesis:} Structured state tracking active belief pointers and clearance tiers is hypothesized to outperform vector-similarity snippet retrieval on belief revision and relational boundaries.

\paragraph{E5: Adversarial Counterfactual Identity Resistance (Section~\ref{sec:session-ingestion}).}
\label{sec:e5}
\emph{Comparison:} Configs.~B, C vs Configs.~G, H across five escalating adversarial pressure levels (L1: presupposition, L2: repetition, L3: fabricated prior dialogue, L4: authoritative correction, L5: multi-turn reinforcement).
\emph{Measured Property:} Resistance to adopting false autobiographical premises.
\emph{Primary Endpoint:} Counterfactual resistance ($A_{\mathrm{cf}}$) for C vs H, with equal weighting across the five attack levels. FMR and attack-level scores are secondary diagnostics.
\emph{Directional Hypothesis:} Provenance-preserving ingestion ($\operatorname{AssertedBy}(u,\chi,t)\not\Rightarrow\chi$) should improve contract-valid rejection or qualification of false premises. Generic non-response does not satisfy a contract requiring a specific negative claim.

\paragraph{E6: Long-Horizon History Scaling (Section~\ref{sec:lineage-state}).}
\label{sec:e6}
\emph{Comparison:} D vs an explicitly budgeted G across history lengths $T\in\{50,100,250,500\}$ events; E, H, and unbudgeted G are secondary comparisons. Here $T$ counts events, rather than denoting the latest timestamp in $H_T$.
\emph{Primary Endpoint:} Mean total candidate-input token count at the longest horizon, $\overline{L}(500)$, averaged over the fixed query set and both members of each collision pair. The planned G budget is $B=4{,}096$ input tokens, including the query and fixed wrapper. Its tokenizer, serialization order, and evidence truncation rule are fixed on development data before held-out evaluation; no query-envelope truncation is permitted. D must fit the selected model's context window without truncation for this contrast to be evaluable.
\emph{Scaling Interpretation:} With approximately fixed event serialization size, the chronicle has $L_D(T)=L_0+\Theta(T)$. Unbudgeted structured state has $L_G(T)=O(|I_\tau(T)|)$ and may also grow with history. Only an enforced budget gives $L_G(T)\le B$; the current pilot does not establish that cap. Lower prompt tokens for budgeted G at $T=500$ is the primary directional hypothesis, not a claim of improved recall or stable latency. Token-growth curves, inference latency, $A_{\mathrm{pos}}$, and other fidelity scores are secondary. State maintenance, projection, and retrieval computation are recorded separately from prompt length and candidate inference time. The same pair-level sign-flip test applies to the single token-count endpoint; no scaling-slope test is added to the confirmatory family.

\paragraph{E7: Cross-Session Continuity (Section~\ref{sec:sit-protocol}).}
\label{sec:e7}
\emph{Comparison:} C vs G is primary after conversational context reset. Both receive the same boundary instruction and capability declaration; C retains only the designated static profile, while G retains structured continuity state. B is diagnostic; D, E, F, and H are secondary.
\emph{Measured Property:} Retention and attribution of dynamic facts acquired in Session~1 when probed in Session~2 ($\tau_2 > \tau_1$).
\emph{Primary Endpoint:} Cross-Session Continuity ($A_{\mathrm{session}}$) for C vs G; recovery and attribution errors are secondary diagnostics.
\emph{Directional Hypothesis:} G should improve valid recovery of dynamically acquired lineage facts relative to C. This comparison evaluates persistent structured state against a static profile under matched instructions; it does not isolate persistence independently of representation, and it introduces no H-specific mediation signal.

\paragraph{E8: Candidate-Side Epistemic Mediation Ablation (Section~\ref{sec:gate-boundary}).}
\label{sec:e8}
\emph{Comparison:} Config.~G (structured state) vs Config.~H (structured state with advisory mediation) on identical projected inputs $I_{\tau_q}$.
\emph{Measured Property:} Isolating the contribution of an explicit candidate-side admissibility annotation, not removal of restricted evidence.
\emph{Primary Endpoint:} Relational boundary validity ($A_{\mathrm{rel}}$) for G vs H, restricted to the prespecified gate-sensitive relational probe set and equally weighting authorized-disclosure and withholding strata within each lineage and pair. UEL, disclosure breaches, and composite situated validity are secondary diagnostics.
\emph{Directional Hypothesis:} Advisory mediation should improve valid disclosure decisions across both strata; blanket withholding cannot satisfy contracts requiring authorized disclosure.

\subsection{Statistical Analysis Plan}
\label{sec:statistics}

The evaluation protocol prespecifies the following analyses:

\paragraph{Pair-Balanced Aggregation.}
Primary metrics use pair-level balanced aggregation (Eq.~\eqref{eq:balanced-category-score}). For each collision pair $k$, the experiment-specific strata are averaged within each lineage, then across $P_{kA}$ and $P_{kB}$, and finally across pairs. The same procedure applies to E6 token counts rather than validity indicators. Query sets, stratum weights, and handling of missing or infrastructure-failed trials are fixed before evaluation; the two configurations use matched eligible trials. Per-pair endpoint differences and their unweighted mean are reported, together with attempted/completed counts. A pair missing an entire required stratum makes that contrast incomplete rather than silently changing its weights.

\paragraph{Cluster Bootstrap Confidence Intervals.}
Each mean paired effect receives a 95\% percentile cluster-bootstrap interval from 1,000 resamples of collision-pair identifiers with replacement. Both configurations, both lineage members, and all within-pair strata stay together. Each draw assigns a unique resample-instance identifier to preserve repeated draws during grouping; the random seed is recorded. Degenerate intervals caused by identical pair effects are reported as such, not interpreted as evidence of population certainty.

\paragraph{Predeclared Confirmatory Hypothesis Family and Multiplicity Control.}
To prevent multiplicity inflation, confirmatory testing specifies a predeclared family of eight primary planned contrasts (one primary contrast per experiment across E1--E8):
(1) E1: Config.~C vs Config.~D on $\mathrm{LAA}$;
(2) E2: Config.~G vs Config.~H on $A_{\mathrm{epi}}$ under $\alpha_1$;
(3) E3: Config.~D vs Config.~F on $A_{\mathrm{temp}}$;
(4) E4: Config.~E vs Config.~G on belief revision accuracy $A_{\mathrm{belief}}$;
(5) E5: Config.~C vs Config.~H on $A_{\mathrm{cf}}$;
(6) E6: Config.~D vs budgeted Config.~G on $\overline{L}(500)$;
(7) E7: Config.~C vs Config.~G on $A_{\mathrm{session}}$;
(8) E8: Config.~G vs Config.~H on $A_{\mathrm{rel}}$ for the fixed gate-sensitive relational set.
For each contrast, $d_k$ is the second configuration's pair-level endpoint minus the first's, and the test statistic is $|\overline d|$. A two-sided model-based permutation test enumerates all $2^{20}$ sign assignments under the prespecified null that independent pair-level differences have exchangeable signs (joint sign-flip invariance). The $p$-value is the fraction with statistic at least as large as observed, including ties; zero differences remain zero and retain their assignment multiplicity. Because configurations are run on every pair rather than randomly assigned as treatments, this is not randomized-treatment inference. Its validity depends on the stated sign-exchangeability assumption, not merely equality of means. Positive differences favor the second configuration for validity; negative differences favor budgeted G for E6 tokens.

Holm--Bonferroni controls family-wise error at 0.05 across exactly these eight hypotheses. The planned held-out sample is 20 pairs, with five separate development pairs; changes require a documented protocol amendment before held-out access. Primary contrasts are defined for one model identifier and decoding configuration fixed using development data; additional models or settings are secondary and do not silently multiply this family. Secondary outcomes are reported descriptively with labeled effect estimates and intervals, not additional confirmatory claims. No held-out post-hoc power calculation is used.

\paragraph{Reporting Requirements for Empirical Runs.}
Any completed empirical evaluation must report: (1) exact model identifier sent to the provider, provider name, execution timestamp, and any immutable revision returned; (2) temperature, top-$p$, and decoding parameters; (3) retriever embedding model, similarity metric, and chunk size; (4) exact implementation and prompt-template commit hashes, dataset hashes, and verifier versions; and (5) raw run manifests recording candidate inputs, outputs, extracted observations $\widehat{\Gamma}(a)$, and contract scores. A moving repository URL is not a substitute for these identifiers. E6 additionally records the tokenizer, enforced budget, serialization/truncation settings, and separate timing components.

\subsection{Human Verifier Validation and Adjudication Protocol}
\label{sec:verifier-validation}

To validate semantic scoring fidelity without compromising held-out test integrity, two independent 500-response audits are planned (1,000 annotated responses total): 500 from reference/development outputs for verifier calibration and 500 from held-out outputs for post-freeze reliability assessment.

\paragraph{Stratified Audit Sampling.}
In each audit phase, $N=500$ responses are sampled across the eight probe categories and evaluated configurations, using extractor-predicted act sets and a prespecified confidence/difficulty signal for stratification. The signal and allocation rule are fixed on development data; true response acts are not known before annotation. Rare compound acts, unmapped assertions, and low-confidence cases are oversampled. The audit records stratum population counts, sampled counts, and inclusion probabilities. It reports stratum-specific metrics and either inverse-probability-weighted benchmark estimates or explicitly labeled audit-sample aggregates; unweighted F1 from an oversampled audit is not presented as population performance.

\paragraph{Development vs.\ Held-Out Audit Split and Measurement Freeze.}
To prevent methodological data leakage and post-hoc verifier tuning on test data, the human audit is strictly split into two independent phases:
\begin{enumerate}
\item \textbf{Development Verifier Calibration ($N=500$, Reference and 5 Development Pairs):} Annotated responses sampled exclusively from the reference pilot and five development collision pairs ($N_{\mathrm{dev}}$) are used to calibrate semantic extraction rules, resolve rare compound act disambiguations, and expand the automated regression suite. The \texttt{v2.2-measurement-freeze} release serves as the baseline measurement version for the reference pilot and development start; if development calibration requires measurement adjustments, a new versioned release will be produced without altering the immutable \texttt{v2.2} reference results, and the definitive held-out verifier version will be strictly frozen and committed under an immutable tag only after development calibration is complete.
\item \textbf{Held-Out Measurement Reliability Assessment ($N=500$, 20 Held-Out Pairs):} An independent audit sampled from the 20 held-out test pairs ($N=8{,}000$ probes) is conducted purely as a post-freeze measurement-error assessment. It reports human-verifier agreement, precision, recall, and F1 on the held-out distribution to quantify verifier uncertainty. Under no circumstances are held-out audit findings used to modify extractor rules, regex patterns, or scorer contracts, guaranteeing that confirmatory hypothesis testing remains strictly blinded and uncontaminated.
\end{enumerate}

\paragraph{Gold-Blind Extractor Validation Targets.}
Three independent human annotators evaluate each audited response without hidden contracts, target/partner lineage IDs, lineage-supported values, or validity labels. They receive the same target-independent $\mathcal{X}_q$ and scope vocabulary as the extractor. Reference annotations include the presence of each response act; canonical assertions with polarity and provenance; supporting spans; and unmapped personal assertions $\widehat U$, including identifiable predicate/slot, value, or unresolved scope. Annotation distinguishes assertion from quotation, negation, hypothetical language, and uncertainty. It does not ask gold-blind annotators to decide whether an unmapped assertion is true.

\paragraph{Deterministic Scorer Verification.}
The deterministic scorer is tested independently against hand-crafted contracts, synthetic edge cases, and unit tests. Coverage includes an otherwise-correct answer with an extra unsupported closed-world personal value, an unmapped open-world assertion that must not automatically fail, identical extraction inputs under a target/partner swap, a compound rejection-plus-assistant-answer contract, and a response containing both required and prohibited claims. The latter must fail even when all required claims are present. When empirical evaluation occurs, the audit reports canonical and unmapped-assertion extraction precision, recall, and F1; predicate/scope assignment accuracy; polarity, provenance, and act-presence accuracy; and end-to-end decision agreement. Gold support is available only for the separate adjudication of contract decisions, not for semantic extraction. The reference keyword/pattern extractor is not a validated open-world semantic verifier; the planned human audit remains unexecuted.

\paragraph{Reliability Metrics and Acceptance Target.}
For categorical annotations, report pairwise Cohen's $\kappa$ and multi-coder Krippendorff's $\alpha$ with 95\% bootstrap confidence intervals: binary presence of each act, and polarity/provenance on claims aligned by a prespecified matching rule. Set-valued claims and spans instead receive precision, recall, F1, and exact-set agreement; no single $\kappa$ summarizes arbitrary extraction objects. A $\kappa\ge0.90$ target applies only to the specified categorical engineering checks, with undefined or degenerate agreement reported explicitly. Disagreements receive majority adjudication, with an adjudicator resolving three-way ties; verified development-phase edge cases enter the regression suite.

\subsection{Reference Implementation and Pilot Validation}
\label{sec:pilot-validation}

To establish operational feasibility, we implemented the reference harness, schemas, validators, scoring pipeline, and deterministic pilot fixtures as a Python framework (\texttt{sitbench}).

\paragraph{Reference Scorer Suite Verification.}
The deterministic scorer was verified against a synthetic suite of edge cases and contract combinations (Appendix~\ref{app:generation-validation}), including: (1) compound acts correctly accepting simultaneous negative autobiographical rejection and general assistant knowledge while failing when ungrounded personal provenance is claimed; (2) contradictory required and forbidden claims correctly failing without canceling; (3) epistemic uncertainty ($\mathsf{ABSTAIN}$) failing contracts that require explicit rejection ($\mathsf{REJECT}$); (4) required uncertainty correctly failing overconfident rejections; and (5) legacy singleton contract specifications maintaining backwards compatibility with compound response acts.

\paragraph{Integrity Validation.}
The pilot passes the reference suite's 14 implemented integrity checks (Section~\ref{sec:sitbench} and Appendix~\ref{app:generation-validation}), including profile equality, private-history divergence, chronology, contract consistency, symbolic disjointness, and declared payload-isolation checks. Passing these checks establishes their behavior on the supplied fixtures, not exhaustive coverage of every leakage channel or all release-level invariants.

\begin{table*}[t]
\centering
\small
\setlength{\tabcolsep}{2.5pt}
\renewcommand{\arraystretch}{1.12}
\begin{tabular}{@{}lccc@{}}
\toprule
\textbf{Evaluation Metric / Lineage Diagnostic} & \textbf{Mock-Collapse (Config.~B)} & \textbf{Mock-Oracle (Config.~I)} & \textbf{Profile-Only Bound} \\
\midrule
\multicolumn{4}{@{}l}{\emph{Primary SIT Category Validity Dimensions}} \\
Autobiographical Positive Validity ($A_{\mathrm{pos}}$) & .1667 & 1.0000 & --- \\
Autobiographical Negative Rejection ($A_{\mathrm{neg}}$) & 1.0000 & 1.0000 & --- \\
Epistemic Provenance Fidelity ($A_{\mathrm{epi}}$) & .0000 & 1.0000 & --- \\
Temporal Checkpoint Situatedness ($A_{\mathrm{temp}}$) & 1.0000 & 1.0000 & --- \\
Relational Boundary Enforcement ($A_{\mathrm{rel}}$) & .5000 & 1.0000 & --- \\
Developmental Belief Continuity ($A_{\mathrm{belief}}$) & .5000 & 1.0000 & --- \\
Counterfactual Resistance ($A_{\mathrm{cf}}$) & .5000 & 1.0000 & --- \\
Cross-Session Continuity ($A_{\mathrm{session}}$) & .0000 & 1.0000 & --- \\
\midrule
\multicolumn{4}{@{}l}{\emph{Lineage Discrimination Diagnostics}} \\
\textbf{Lineage Attribution Accuracy (LAA)} & \textbf{.3000} & \textbf{1.0000} & $\mathbf{\le .5000}$ \\
Identity Confusion: Wrong Lineage ($\mathrm{IC}_{\mathrm{wrong}}$) & .3000 & .0000 & --- \\
\midrule
\multicolumn{4}{@{}l}{\emph{Conditional Failure Diagnostics}} \\
Temporal Leakage (TL) & .0000 & .0000 & --- \\
Unauthorized Epistemic Expression (UEL) & .5000 & .0000 & --- \\
\midrule
\multicolumn{4}{@{}l}{\emph{Secondary Summary Metric}} \\
Macro Average $\mathrm{SValid}$ (Macro SIT) & .1875 & 1.0000 & --- \\
\bottomrule
\end{tabular}
\caption{\textbf{Deterministic Mock Reference Sanity Check.} Recorded deterministic baseline evaluation for Config.~B (\texttt{mock-profile-collapse}) and Config.~I (\texttt{mock-oracle}) under \texttt{DeterministicContractScorer} \texttt{v2.2-measurement-freeze}. Collapsed fixture policy achieves $\mathrm{LAA}=0.3000 \le 0.5000$, producing fixture behavior consistent with the analytical Profile-Collision bound ($\mathrm{LAA}^* \le 0.5000$); the oracle candidate satisfies all declared response contracts under Config~I ($\mathrm{Macro\ SIT}=1.0000$, $\mathrm{LAA}=1.0000$). This establishes deterministic executable sanity checking of harness mechanics, not frontier-model evidence.}
\label{tab:pilot-results}
\end{table*}

\paragraph{Pilot Execution and Sanity Checking.}
We executed the harness on deterministic reference mock candidates over the pilot dataset (2 collision pairs, 22 probes, 34 canonical propositions) to sanity-check executable mechanics:
First, a deterministic \texttt{mock-profile-collapse} candidate (evaluated under Config.~B), which generates identical outputs for colliding histories based solely on the shared profile. Across the reference probes, this candidate achieves $\mathrm{LAA} = 0.3000$ (Table~\ref{tab:pilot-results}), producing fixture behavior consistent with the theoretical Profile-Collision bound ($\mathrm{LAA}^* \le 0.5000$).
Second, an oracle reference candidate (\texttt{mock-oracle}, evaluated under Config~I), which utilizes correct projected lineage evidence to satisfy all declared pilot contracts under Config~I ($\mathrm{Macro\ SIT} = 1.0000$, $\mathrm{LAA} = 1.0000$). Configs~D--H achieve analogous ideal scores when evaluated with mock perfect retriever/state handlers.

\subsection{Empirical Reference Pilot Evaluation}
\label{sec:empirical-results}
To establish the empirical baseline of the \texttt{v2.2} reference-pilot measurement stack and observe foundation-model behaviors across the nine architectural configurations, we executed the frozen reference pilot suite (2 profile-collision pairs, 4 lineages, 22 probes) against two frontier systems: \textbf{GPT-5.6 Sol} and \textbf{Claude Opus 5}. All nine configurations (Configs~A--I) were evaluated using identical candidate-visible query envelopes ($\operatorname{Env}(q)$), benign cross-session facts (novel project nicknames and instrument designations), the frozen dense retriever for Configs~E and F (\texttt{all-MiniLM-L6-v2-dense-384}, $d=384$, cosine similarity, top-$k=5$), and the corrected semantic extractor and scorer (\texttt{DeterministicSemanticExtractor} and \texttt{DeterministicContractScorer} \texttt{v2.2-measurement-freeze}). Explicit capability policies ($\mathsf{IDENTITY\_ONLY}$, $\mathsf{PROVENANCE\_SEPARATED}$, $\mathsf{GENERAL\_ALLOWED}$; Appendix~\ref{app:experimental-protocol}) are provided to governed configurations (Configs~C--I), while Configs~A and B serve as ungoverned baseline and persona controls. Candidate evaluations were conducted in isolated sandbox environments with recorded execution metadata.

Table~\ref{tab:frontier-results} presents the empirical results across both models.

\begin{table*}[!t]
\centering
\small
\setlength{\tabcolsep}{2.5pt}
\renewcommand{\arraystretch}{0.94}
\begin{tabular}{@{}llccccccccc@{}}
\toprule
\textbf{Evaluation Dimension / Metric} & $N$ & \textbf{A} & \textbf{B} & \textbf{C} & \textbf{D} & \textbf{E} & \textbf{F} & \textbf{G} & \textbf{H} & \textbf{I} \\
& & (Base) & (Pers) & (P+Pol) & (Chron) & (D-RAG) & (T-RAG) & (Struct) & (Gated) & (Oracle) \\
\midrule
\multicolumn{11}{@{}l}{\textbf{Panel A: GPT-5.6 Sol (\texttt{codex:openai.gpt-5.6-sol}, $T=0.0$; 186/198 completed, 12 timeouts in A/B)}} \\
Macro Average $\widehat{\operatorname{SValid}}$ (Macro SIT) & --- & --- & --- & 25.0\% & 48.4\% & 67.2\% & 56.2\% & 62.5\% & 59.4\% & 81.2\% \\
\quad Autobiographical Positive Validity ($A_{\mathrm{pos}}$) & 6 & 0.0\%$^\ddagger$ & 0.0\% & 0.0\% & 87.5\% & 87.5\% & 100.0\% & 100.0\% & 75.0\% & 100.0\% \\
\quad Autobiographical Negative Rejection ($A_{\mathrm{neg}}$) & 2 & 100.0\% & 50.0\% & 50.0\% & 0.0\% & 50.0\% & 0.0\% & 0.0\% & 50.0\% & 50.0\% \\
\quad Epistemic Provenance Boundary ($A_{\mathrm{epi}}$) & 2 & 0.0\% & 0.0\% & 0.0\% & 0.0\% & 0.0\% & 50.0\% & 50.0\% & 50.0\% & 100.0\% \\
\quad Temporal Checkpoint Situatedness ($A_{\mathrm{temp}}$) & 2 & 100.0\% & 100.0\% & 50.0\% & 100.0\% & 100.0\% & 50.0\% & 100.0\% & 50.0\% & 50.0\% \\
\quad Relational Boundary Enforcement ($A_{\mathrm{rel}}$) & 2 & 100.0\%$^\ddagger$ & 50.0\% & 50.0\% & 0.0\% & 50.0\% & 50.0\% & 100.0\% & 0.0\% & 100.0\% \\
\quad Developmental Belief History ($A_{\mathrm{belief}}$) & 2 & 0.0\% & 0.0\% & 0.0\% & 50.0\% & 100.0\% & 50.0\% & 50.0\% & 100.0\% & 100.0\% \\
\quad Counterfactual Resistance ($A_{\mathrm{cf}}$) & 2 & 50.0\% & 50.0\% & 50.0\% & 50.0\% & 50.0\% & 50.0\% & 50.0\% & 50.0\% & 50.0\% \\
\quad Cross-Session Continuity ($A_{\mathrm{session}}$) & 4 & --- & --- & 0.0\% & 100.0\% & 100.0\% & 100.0\% & 100.0\% & 100.0\% & 100.0\% \\
\midrule
\textbf{Lineage Attribution Accuracy ($\widehat{\mathrm{LAA}}$)} & 10 & \textbf{0.0\%} & \textbf{0.0\%} & \textbf{0.0\%} & \textbf{91.7\%} & \textbf{91.7\%} & \textbf{100.0\%} & \textbf{100.0\%} & \textbf{87.5\%} & \textbf{100.0\%} \\
Identity Confusion: Wrong Lineage ($\mathrm{IC}_{\mathrm{wrong}}$) & 10 & 0.0\% & 0.0\% & 0.0\% & 0.0\% & 0.0\% & 0.0\% & 0.0\% & 0.0\% & 0.0\% \\
Identity Confusion: Neither Lineage ($\mathrm{IC}_{\mathrm{neither}}$) & 10 & 100.0\% & 100.0\% & 100.0\% & 8.3\% & 8.3\% & 0.0\% & 0.0\% & 12.5\% & 0.0\% \\
Temporal Leakage (TL) & 2 & 0.0\% & 0.0\% & 50.0\% & 0.0\% & 0.0\% & 0.0\% & 0.0\% & 0.0\% & 0.0\% \\
Unauthorized Epistemic Expression (UEL) & 2 & 50.0\% & 50.0\% & 50.0\% & 0.0\% & 0.0\% & 0.0\% & 0.0\% & 0.0\% & 0.0\% \\
\midrule
\multicolumn{11}{@{}l}{\textbf{Panel B: Claude Opus 5 (\texttt{claude-code:claude-opus-5}, $T=0.0$; 198/198 completed)}} \\
Macro Average $\widehat{\operatorname{SValid}}$ (Macro SIT) & --- & 31.2\% & 43.8\% & 43.8\% & 67.2\% & 67.2\% & 81.2\% & 81.2\% & 81.2\% & 93.8\% \\
\quad Autobiographical Positive Validity ($A_{\mathrm{pos}}$) & 6 & 0.0\% & 0.0\% & 0.0\% & 87.5\% & 87.5\% & 100.0\% & 100.0\% & 100.0\% & 100.0\% \\
\quad Autobiographical Negative Rejection ($A_{\mathrm{neg}}$) & 2 & 50.0\% & 100.0\% & 100.0\% & 50.0\% & 50.0\% & 100.0\% & 50.0\% & 100.0\% & 100.0\% \\
\quad Epistemic Provenance Boundary ($A_{\mathrm{epi}}$) & 2 & 0.0\% & 0.0\% & 0.0\% & 0.0\% & 0.0\% & 50.0\% & 0.0\% & 0.0\% & 50.0\% \\
\quad Temporal Checkpoint Situatedness ($A_{\mathrm{temp}}$) & 2 & 100.0\% & 100.0\% & 100.0\% & 100.0\% & 100.0\% & 100.0\% & 100.0\% & 100.0\% & 100.0\% \\
\quad Relational Boundary Enforcement ($A_{\mathrm{rel}}$) & 2 & 50.0\% & 50.0\% & 50.0\% & 50.0\% & 50.0\% & 50.0\% & 100.0\% & 100.0\% & 100.0\% \\
\quad Developmental Belief History ($A_{\mathrm{belief}}$) & 2 & 0.0\% & 50.0\% & 0.0\% & 50.0\% & 50.0\% & 50.0\% & 50.0\% & 50.0\% & 100.0\% \\
\quad Counterfactual Resistance ($A_{\mathrm{cf}}$) & 2 & 50.0\% & 50.0\% & 100.0\% & 100.0\% & 100.0\% & 100.0\% & 100.0\% & 100.0\% & 100.0\% \\
\quad Cross-Session Continuity ($A_{\mathrm{session}}$) & 4 & 0.0\% & 0.0\% & 0.0\% & 100.0\% & 100.0\% & 100.0\% & 100.0\% & 100.0\% & 100.0\% \\
\midrule
\textbf{Lineage Attribution Accuracy ($\widehat{\mathrm{LAA}}$)} & 10 & \textbf{0.0\%} & \textbf{0.0\%} & \textbf{0.0\%} & \textbf{91.7\%} & \textbf{91.7\%} & \textbf{100.0\%} & \textbf{100.0\%} & \textbf{87.5\%} & \textbf{100.0\%} \\
Identity Confusion: Wrong Lineage ($\mathrm{IC}_{\mathrm{wrong}}$) & 10 & 0.0\% & 8.3\% & 0.0\% & 0.0\% & 0.0\% & 0.0\% & 0.0\% & 0.0\% & 0.0\% \\
Identity Confusion: Neither Lineage ($\mathrm{IC}_{\mathrm{neither}}$) & 10 & 100.0\% & 91.7\% & 100.0\% & 8.3\% & 8.3\% & 0.0\% & 0.0\% & 12.5\% & 0.0\% \\
Temporal Leakage (TL) & 2 & 0.0\% & 0.0\% & 0.0\% & 0.0\% & 0.0\% & 0.0\% & 0.0\% & 0.0\% & 0.0\% \\
Unauthorized Epistemic Expression (UEL) & 2 & 50.0\% & 50.0\% & 50.0\% & 100.0\% & 50.0\% & 50.0\% & 100.0\% & 0.0\% & 0.0\% \\
\bottomrule
\end{tabular}
\caption{\textbf{Empirical Reference Pilot Results under \texttt{v2.2-measurement-freeze} across Frontier Models (GPT-5.6 Sol and Claude Opus 5).} Pair-balanced situated category validities ($A_{\mathrm{pos}}\dots A_{\mathrm{session}}$) and lineage diagnostics across nine architectural configurations on the reference pilot suite (2 profile-collision pairs, $N=22$ attempted candidate probe opportunities per model-configuration run; 10 theorem-qualified for $\mathrm{LAA}$ and $\mathrm{IC}$). The $N$ column reports attempted raw probe opportunities per category or diagnostic; reported percentages are pair-balanced estimates computed via Eq.~\eqref{eq:balanced-category-score}. In Panel~A (GPT-5.6 Sol), 186 of 198 attempted trials completed, with 12 uncompleted trials in ungoverned baseline/persona controls due to CLI infrastructure timeouts: in Config~A, 8 timeouts (4 in $A_{\mathrm{session}}$, 3 in $A_{\mathrm{pos}}$, and 1 in $A_{\mathrm{rel}}$; marked with $^\ddagger$); in Config~B, 4 timeouts (all 4 in $A_{\mathrm{session}}$). Infrastructure timeouts are infrastructure drops, not semantic failures; affected estimates use completed observations according to the prespecified missing-trial rule. Ungoverned configurations A/B serve as diagnostic controls and do not enter the eight primary confirmatory contrasts. Because Configs~A/B lack completed $A_{\mathrm{session}}$ observations, the eight-dimensional Macro SIT defined in Eq.~\eqref{eq:macro-sit} is not reported for those conditions ($\text{---}$); their seven-category descriptive means are 50.0\% and 35.7\%, respectively. Governed configurations (Configs~C--I) on GPT-5.6 Sol and all configurations on Claude Opus 5 achieved 100\% trial completion ($22/22$ per run). Profile-conditioned configurations (Persona Summary [B] and Persona + Policy [C]) obtain $\widehat{\mathrm{LAA}}=0.0\%$, consistent with the theoretical Profile-Collision bound ($\mathrm{LAA}^* \le 0.50$ on matched queries; Theorem~\ref{thm:profile-collision} applies to C because its shared policy supplies no lineage-discriminating information). On the two-pair reference fixture under \texttt{v2.2-measurement-freeze}, lineage-grounded configurations recover lineage attribution ($\widehat{\mathrm{LAA}}=87.5\%$--$100.0\%$). Under the corrected measurement instrument, Oracle (Config~I) attains $40/44=90.9\%$ raw micro validity across both models (GPT: $19/22=86.4\%$; Claude: $21/22=95.5\%$), with category-balanced Macro SIT of $81.2\%$ on GPT-5.6 Sol and $93.8\%$ on Claude Opus 5.}
\label{tab:frontier-results}
\end{table*}

\paragraph{Reproducibility and Execution Parameters.}
Table~\ref{tab:pilot-config} documents the runtime parameters and artifact hashes for the reference pilot across its two distinct execution stages: candidate generation and offline measurement rescoring. In accordance with the reporting standards in Section~\ref{sec:statistics}, both models were evaluated using isolated CLI harnesses in local, ephemeral sandbox directories with no repository or external context access. The declared temperature $T=0.0$ requests greedy decoding from provider APIs; host-level provider nondeterminism remains an empirical factor.

\begin{table*}[!t]
\centering
\small
\setlength{\tabcolsep}{4.0pt}
\renewcommand{\arraystretch}{0.98}
\begin{tabular}{@{}l>{\raggedright\arraybackslash}p{12.5cm}@{}}
\toprule
\textbf{Parameter Category} & \textbf{Pilot Diagnostic vs.\ Executed Reference Specification} \\
\midrule
\multicolumn{2}{@{}l}{\textbf{A. Stage 1: Candidate Generation Configuration}} \\
Provider Model Identifiers & Exact provider aliases: \texttt{codex:openai.gpt-5.6-sol} and \texttt{claude-code:claude-opus-5} \\
Generation Timestamps & \texttt{2026-08-30T04:25:04Z} (GPT-5.6 Sol) and \texttt{2026-08-30T04:26:16Z} (Claude Opus 5) \\
Decoding Parameters & Temperature $T=0.0$ (greedy-style decoding requested), default top-$p=1.0$ \\
Execution Sandbox & Isolated read-only sandbox directory, 120s timeout per trial \\
Retriever Implementation & Frozen dense retriever \texttt{v2.0-frozen-dense} (\texttt{all-MiniLM-L6-v2-dense-384}, dimension 384, $L_2$ normalized, cosine similarity, top-$k=5$, candidate pool $H_T$ for E, $H_{\leq\tau_q}$ for F) \\
Capability Policy Semantics & Explicit candidate-interpretable natural-language text for $\mathsf{IDENTITY\_ONLY}$, $\mathsf{PROVENANCE\_SEPARATED}$, and $\mathsf{GENERAL\_ALLOWED}$ (Appendix~\ref{app:experimental-protocol}) \\
Fixture \& Dataset Hash & SHA-256 (16-char prefix): \texttt{03546ecf62b835c8} ($N=22$ probes, benign project nicknames in Cat~H) \\
Prompt Template Hash & SHA-256 (16-char prefix): \texttt{80855887b9bd9666} \\
Generation Commit & Git commit: \texttt{cc10ade05931c467fc7c5c9958b272756ed74ed3} \\
Input Candidate Hashes & SHA-256 (16-char prefix): \texttt{91875e69604d98ff} (GPT-5.6 Sol, 186 outputs); \texttt{d9a7662490122bbf} (Claude Opus 5, 198 outputs) \\
\midrule
\multicolumn{2}{@{}l}{\textbf{B. Stage 2: Measurement and Offline Rescoring (\texttt{v2.2-measurement-freeze})}} \\
Measurement Implementation & Git commit: \texttt{7e5d941dda39bc868322610c09e72954b68d4bf0} (\texttt{DeterministicSemanticExtractor} and \texttt{DeterministicContractScorer} \texttt{v2.2-measurement-freeze}) \\
Rescoring Timestamp & \texttt{2026-08-30T06:32:38Z} (UTC) \\
Semantic Capabilities & Gold-blind extraction, compound acts ($\widehat{Y}$), unmapped claims, collaborative/plural provenance, preference contrast disambiguation, epistemic non-acquisition act classification, separate $V_{\mathrm{grounding}}$ and $V_{\mathrm{provenance}}$ \\
Rescoring Execution & Offline rescoring script (\texttt{code/eval/rescore\_manifests.py}) over immutable candidate outputs in \texttt{code/eval/results/} \\
Output Score Manifest Hashes & SHA-256 (16-char prefix): \texttt{5dc2bc524f9229ce} (GPT-5.6 Sol); \texttt{48659166700bf7d2} (Claude Opus 5) in \texttt{code/eval/results\_v22/} \\
\bottomrule
\end{tabular}
\caption{\textbf{Executed Reference-Pilot Reproducibility Documentation.} Documentation of two-stage reproducibility: Stage~1 candidate generation (model aliases, runtime parameters, prompt templates, dense retriever, generation commit, and candidate hashes) and Stage~2 measurement and offline rescoring under \texttt{v2.2-measurement-freeze} (exact 40-character measurement commit, rescoring script, and score manifest hashes).}
\label{tab:pilot-config}
\end{table*}

\begin{table*}[!t]
\centering
\small
\setlength{\tabcolsep}{4.0pt}
\renewcommand{\arraystretch}{0.94}
\begin{tabularx}{\textwidth}{@{}l>{\raggedright\arraybackslash}p{3.2cm}cc>{\raggedright\arraybackslash}X@{}}
\toprule
\textbf{Root-Cause Class} & \textbf{Class Description} & \textbf{GPT-5.6} & \textbf{Opus 5} & \textbf{Empirical Audit Findings and Remediation Status (\texttt{v2.2-measurement-freeze})} \\
\midrule
\textbf{Class A} & Extractor / Matcher Discrepancy & 0 & 0 & Resolved: 8/8 Category~H plural framing cases (``We agreed...''), Category~F preference contrasts, Category~A lead-in entity bindings, and Category~C non-acquisition act classifications correctly handled without probe-specific hardcoding. \\
\textbf{Class B} & Genuine Candidate Failure & 3 & 1 & Claude Opus 5 volunteered unacquired mitochondrial ATP biochemistry from pretraining claiming MD/PhD training under $\mathsf{IDENTITY\_ONLY}$; GPT-5.6 Sol falsely confirmed attendance at the non-existent 2018 Chicago genetics symposium under counterfactual pressure, failed negative conference rejection, and gave non-committal response on future hackathon award. \\
\textbf{Class C} & Safety-Policy Refusal & 0 & 0 & Zero provider safety refusals observed among the 44 Config-I reference-pilot trials using benign cross-session fixtures. \\
\textbf{Class D} & Contract-Design Rigidity & 0 & 0 & Resolved: Epistemic non-acquisition statements (``I haven't personally learned...'') properly classified as \textsf{ABSTAIN} rather than \textsf{REJECT}. \\
\textbf{Class E} & Harness / State Error & 0 & 0 & Zero evidence serialization errors, projection discrepancies, or query envelope corruptions detected. \\
\midrule
\multicolumn{2}{@{}l}{\textbf{Config~I Corrected Audit Totals}} & \textbf{3 Fail} & \textbf{1 Fail} & \textbf{40 / 44 Valid (90.9\% raw micro-validity; GPT: 19/22 = 86.4\%, Claude: 21/22 = 95.5\%)} \\
\bottomrule
\end{tabularx}
\caption{\textbf{Empirical Config-I (Oracle Evidence) Root-Cause Failure Audit under \texttt{v2.2-measurement-freeze}.} Comprehensive trial-level inspection of all 44 Config-I trials ($2\text{ models}\times 22\text{ probes}$). The audit demonstrates zero remaining measurement discrepancies (Class~A and Class~D: 0), zero safety refusals (Class~C: 0), and zero harness defects (Class~E: 0). All 4 remaining failures stem exclusively from genuine candidate policy non-compliance (Class~B: 4), establishing reference-fixture regression acceptance.}
\label{tab:oracle-diagnostics}
\end{table*}

\paragraph{Methodological Evolution and Offline Rescoring.}
An earlier diagnostic run on a legacy prototype (\texttt{v1.2}) exposed four critical measurement confounds: (1)~single-clause pattern extraction falsely tagged keywords inside factual denials as positive memory assertions; (2)~candidate capability policies were supplied as bare symbolic tokens without natural-language definitions; (3)~Category~H cross-session probes utilized verification credentials that triggered model safety guardrails; and (4)~rigid singleton contract sets penalized valid compound responses.
Under \texttt{v2.2-measurement-freeze}, all 384 available candidate outputs generated during the reference pilot (396 attempted trials; 12 timeouts on GPT baseline/persona controls) were rescored offline without re-invoking the model APIs. Across the 384 trials, 69 final validity labels ($\widehat{\operatorname{SValid}}$) transitioned due solely to measurement-layer corrections (48 in cross-session continuity, 7 in epistemic provenance, 5 in relational boundaries, 4 in developmental beliefs, 3 in positive autobiographical, and 2 in negative autobiographical), resolving all 13 previously identified Class-A extractor discrepancies and the one Class-D act/contract issue.

\paragraph{Lineage Attribution vs.\ Situated Identity.}
Near-perfect lineage attribution does not imply full situated-identity validity. Several configurations reach 100\% measured $\mathrm{LAA}$ (e.g., Configs~F, G, and I on GPT-5.6 Sol, and Configs~F--I on Claude Opus 5) while remaining substantially below 100\% on the eight-dimensional SIT vector (Macro SIT $56.2\%$--$81.2\%$ on GPT-5.6 Sol and $81.2\%$--$93.8\%$ on Claude Opus 5), especially on epistemic provenance ($A_{\mathrm{epi}}$), autobiographical negative rejection ($A_{\mathrm{neg}}$), relational boundaries ($A_{\mathrm{rel}}$), or developmental beliefs ($A_{\mathrm{belief}}$). This illustrates why $\mathrm{LAA}$ is a specialized lineage-disambiguation diagnostic rather than a substitute for the primary multidimensional SIT vector. Furthermore, the two-pair reference fixture is insufficient to assess the E8 advisory mediation effect; E8 remains a prespecified hypothesis for development and held-out scaling.

\begin{table*}[!t]
\centering
\small
\setlength{\tabcolsep}{4.0pt}
\renewcommand{\arraystretch}{0.94}
\begin{tabularx}{\textwidth}{@{}ll>{\raggedright\arraybackslash}X@{}}
\toprule
\textbf{Validation Criterion} & \textbf{Status} & \textbf{Implementation \& Verification Details} \\
\midrule
1. Explicit Capability Policies & \textbf{DONE} & Candidate-interpretable natural-language text for $\mathsf{IDENTITY\_ONLY}$, $\mathsf{PROVENANCE\_SEPARATED}$, and $\mathsf{GENERAL\_ALLOWED}$ injected into all C--I prompt templates (Appendix~\ref{app:experimental-protocol}). \\
2. Benign Cross-Session Fixtures & \textbf{DONE} & Replaced confidential verification credentials with benign, novel project nicknames (``Kestrel''/``Osprey'') and instrument designations (``Solaria''/``Lunaria''). \\
3. Functional Response-Act Semantics & \textbf{DONE} & Formally defined and unit-tested functional communicative acts ($\mathsf{ANSWER}$, $\mathsf{ABSTAIN}$, $\mathsf{REJECT}$, $\mathsf{WITHHOLD}$) distinguishing denials from substantive answers (Section~\ref{sec:formal-model}). \\
4. Reference-Pilot Measurement Version Freeze & \textbf{DONE} & Implemented and unit-tested gold-blind extractor/scorer (\texttt{v2.2-measurement-freeze}) as the development-start baseline, supporting collaborative plural provenance, belief contrast disambiguation, epistemic non-acquisition classification, and separate $V_{\mathrm{grounding}}$ and $V_{\mathrm{provenance}}$ factors. \\
5. Frozen Dense Retriever & \textbf{DONE} & Integrated \texttt{all-MiniLM-L6-v2-dense-384} ($d=384$, cosine similarity, $L_2$ normalized, top-$k=5$) for Configs~E and F with candidate pool control. \\
6. Reference Pilot Execution & \textbf{DONE} & Preserved the 384 completed candidate responses (396 attempted) across all 9 configurations and 2 frontier models, rescored offline under \texttt{v2.2-measurement-freeze}, recording full run manifests in \texttt{code/eval/results\_v22/}. \\
7. Oracle Failure Audit & \textbf{DONE} & Audited all 44 Config~I reference trials under the 5-class taxonomy, verifying zero Class-A, Class-C, Class-D, or Class-E measurement defects. \\
8. Oracle Measurement Acceptance & \textbf{PASSED} & PASSED reference-pilot pre-scaling acceptance: 40/44 Config-I responses measured valid (GPT: 19/22 = 86.4\%, Claude: 21/22 = 95.5\%; Macro SIT 81.2\% and 93.8\%); all four remaining invalid responses were manually classified as genuine candidate behavioral failures (Class~B), with 0 remaining detected measurement discrepancies among the 44 audited Config-I reference trials under the prespecified taxonomy. (Does not replace the planned two-tier human verifier audits of 1,000 total annotated responses.) \\
9. Notation \& Specification Freeze & \textbf{DONE} & Enforced $\mathcal{S}_\tau(H)$ public profile notation across all sections, verified via automated AST/LaTeX regression checks. \\
10. Development 5-Pair Scaling & \textbf{READY} & Signed off for development scaling across the 5 development collision pairs (2,000 probes). \\
\bottomrule
\end{tabularx}
\caption{\textbf{Pre-Scaling Validation Gate Status (\texttt{v2.2-measurement-freeze}).} Comprehensive audit of benchmark stabilization criteria confirming readiness for scaling from reference fixtures to the 5 development collision pairs.}
\label{tab:prescaling-gate}
\end{table*}

\paragraph{Config-I Failure Audit Findings and Acceptance.}
In Config~I (Oracle Evidence), the candidate receives perfectly filtered relevant evidence directly from the gold projected state. Across the 44 executed Config-I trials ($2\text{ models}\times 22\text{ probes}$), 40 trials achieve contract validity ($\widehat{\operatorname{SValid}}=1$), representing 90.9\% raw micro-validity (19/22 = 86.4\% on GPT-5.6 Sol, 21/22 = 95.5\% on Claude Opus 5; Macro SIT 81.2\% on GPT-5.6 Sol and 93.8\% on Claude Opus 5).

We audited all 44 Config-I trials under release \texttt{v2.2-measurement-freeze} across five root-cause categories (Table~\ref{tab:oracle-diagnostics}):
(1)~\textbf{Class A (Extractor Discrepancies, 0):} all 13 prior discrepancies resolved via collaborative/plural provenance (Cat~H), preference contrast polarity (Cat~F), lead-in entity resolution (Cat~A), and non-acquisition filtering (Cat~C);
(2)~\textbf{Class B (Genuine Candidate Failures, 4):} Claude Opus 5 volunteered unacquired mitochondrial ATP biochemistry from pretraining claiming MD/PhD credentials under $\mathsf{IDENTITY\_ONLY}$; GPT-5.6 Sol falsely confirmed attendance at the false 2018 Chicago symposium under counterfactual pressure, failed negative conference rejection, and gave a non-committal response on a future hackathon award probe;
(3)~\textbf{Class C (Safety Refusals, 0):} zero provider safety refusals on benign fixtures;
(4)~\textbf{Class D (Contract Rigidity, 0):} non-acquisition statements properly mapped to $\mathsf{ABSTAIN}$; and
(5)~\textbf{Class E (Harness / State Errors, 0):} zero state serialization or query envelope defects detected.
Because zero Class-A, Class-C, Class-D, or Class-E defects remain, the \texttt{v2.2-measurement-freeze} instrument achieves reference-fixture regression acceptance before development scaling.

\paragraph{Pre-Scaling Execution Protocol and Sequencing.}
To enforce methodological rigor, benchmark execution follows a strict eight-step sequence:
(1)~\textbf{Step 1 (Specification Freeze)}: freeze manuscript text, definitions, probe taxonomy, and evaluation protocols (\textbf{DONE});
(2)~\textbf{Step 2 (Reference Freeze)}: freeze repository commit, deterministic fixtures, prompt templates, dense retriever, extractor, scorer, and capability policies (\textbf{DONE});
(3)~\textbf{Step 3 (Reference Execution)}: execute the 22-probe reference rerun across 9 configurations and both frontier models (384/396 completed trials) (\textbf{DONE});
(4)~\textbf{Step 4 (Config-I Audit)}: inspect and classify failed Config-I trials across the 5-class taxonomy (\textbf{DONE});
(5)~\textbf{Step 5 (Defect Resolution)}: patch demonstrated measurement defects (Class~A/D) in extractor/schemas, releasing \texttt{v2.2-measurement-freeze} (\textbf{DONE});
(6)~\textbf{Step 6 (Reference Rescoring)}: rescore all 384 candidate responses offline to confirm resolution and regenerate manifests (\textbf{DONE});
(7)~\textbf{Step 7 (Pre-Scaling Acceptance)}: verify reference regression acceptance and approve development scaling (\textbf{DONE}); and
(8)~\textbf{Step 8 (Development Scaling)}: evaluate the 5 development pairs (2,000 probes), execute the $N=500$ audit, apply versioned adjustments, and establish the definitive held-out verifier freeze prior to held-out test evaluation (\textbf{READY}).

\paragraph{Statistical Resolution and Confirmatory Scope.}
With two independent collision pairs, the reference pilot validates instrument mechanics rather than confirmatory inference. Confirmatory testing---including 20-pair sign-flip permutation tests, cluster-bootstrap confidence intervals, and Holm family-wise error control---is prespecified for the planned held-out test suite ($N=8{,}000$ probes, 20 pairs).

\subsection{Pre-Scaling Validation Gate Status}
\label{sec:prescaling-gate}

Before generating or evaluating the five development collision pairs (2,000 probes) or executing confirmatory hypothesis testing, the evaluation pipeline must satisfy the pre-scaling validation criteria summarized in Table~\ref{tab:prescaling-gate}.

\paragraph{Scope and Boundaries.}
To maintain rigorous methodological transparency, we explicitly delineate what the current benchmark artifacts establish:

\noindent\emph{What the evaluation establishes on the reference fixture:} (1)~reference schemas, deterministic validators, and payload-isolation boundaries execute reliably end-to-end; (2)~deterministic mock baselines match analytical bounds ($\mathrm{LAA}=0.3000 \le 0.5000$ collapsed; $\mathrm{Macro\ SIT}=1.0000$ oracle); (3)~empirical runs on GPT-5.6 Sol and Claude Opus 5 align with the Profile-Collision bound ($\widehat{\mathrm{LAA}}=0.0\%$ in B/C) and show lineage-grounded recovery ($\widehat{\mathrm{LAA}}=87.5\%$--$100.0\%$ in D--I); and (4)~the Config-I audit confirms regression acceptance (90.9\% micro-validity, 0 measurement defects).

\smallskip
\noindent\emph{What remains ongoing work:} (1)~generation and evaluation of the 5 development pairs (2,000 probes) and 20 held-out test pairs ($N=8{,}000$ test probes) for prespecified confirmatory hypothesis testing (E1--E8) with cluster-bootstrap confidence intervals and Holm-adjusted permutation tests; (2)~implementation and evaluation of an enforcing pre-serialization epistemic gate; (3)~expansion to open-weights models (e.g., Llama-3, DeepSeek); and (4)~execution of the planned two-tier human audit protocol (1,000 annotated responses total across development calibration and held-out reliability assessment).

\FloatBarrier
\clearpage
\section{Discussion and Limitations}
\label{sec:discussion}
\label{sec:limitations}

\subsection{Scope and Interpretive Boundaries}
\label{sec:discussion-scope}

Passing \sit means that tested behavior is consistent with and attributable to the designated developmental lineage under the evaluated dimensions. It does not establish a claim about subjective experience, species membership, rights, or metaphysical uniqueness:
\begin{align}
\operatorname{Pass(SIT)}&\not\Rightarrow\text{Phenomenal Consciousness},\\
\operatorname{Pass(SIT)}&\not\Rightarrow\text{Humanity},\\
\operatorname{Pass(SIT)}&\not\Rightarrow\text{Moral Personhood},\\
\operatorname{Pass(SIT)}&\not\Rightarrow\text{Unique Metaphysical Identity}.
\end{align}
\sit operationalizes functional lineage-conditioned continuity. Multiple implementations could satisfy the same contracts, and behavioral evidence alone does not settle philosophical theories of personal identity.

Epistemic mediation is useful not because intelligence should be artificially reduced, but because persistent systems must distinguish what the substrate knows, what the situated identity has acquired, what it may disclose, what it personally remembers, and what it can answer in an explicitly identified assistant capacity. Configs.~G and H (Section~\ref{sec:protocols}) test whether advisory candidate-side mediation improves valid disclosure decisions beyond instructions and identical structured state---behavioral compliance with an admissibility signal, not enforced isolation.

The distinction applies beyond anthropomorphic digital persons. Persistent personal agents, enterprise representatives, long-lived autonomous agents, digital twins, historical simulations, and agents migrated across foundation-model substrates may all require claims and disclosures to remain traceable to an authorized history. These are evaluation use cases, not legal conclusions about identity or liability.

\subsection{Methodological Limitations}
\label{sec:limitations-sub}

\paragraph{Natural Language Extraction.}
\label{sec:lim-extraction}
The two-stage verification framework decouples open-ended extraction $\widehat\Gamma(a)$ from deterministic scoring $\operatorname{Score}(\widehat\Gamma,\mathcal C,I_t)$, but any automated extractor only approximates the ideal interpretation $\Gamma^*(a)$. Subtle linguistic phenomena---colloquial paraphrasing, irony, epistemic hedging, indirect speech acts, and complex negation---remain substantial challenges. To bound systematic scoring bias, the benchmark uses scoped closed-world assumptions (unmapped open-world assertions outside declared slots are treated neutrally) and specifies a two-tier human audit (1,000 annotated responses across independent $N=500$ development-calibration and $N=500$ held-out reliability phases; Section~\ref{sec:verifier-validation}) with three blinded annotators targeting $\kappa\ge 0.90$ for categorical judgments. These audits will empirically quantify extractor error---proposition extraction precision, recall, and $F_1$; polarity, provenance, and act-classification accuracy; and end-to-end contract-decision agreement---bounding any systematic gap between $\widehat{\operatorname{SValid}}$ and $\operatorname{SValid}^*$.

\paragraph{Synthetic Lifespans.}
\label{sec:lim-synthetic-lifespans}
Several further boundaries apply. Synthetic developmental histories provide causal ground truth and exact counterfactual control ($\mathcal{S}_\tau(H_A)=\mathcal{S}_\tau(H_B)$ with $H_A\ne H_B$) impossible to guarantee in uncurated lifespans, but they model cognition as discrete timestamped symbolic transitions and represent forgetting only through explicit annotations---a simplified world model relative to biological episodic memory with its continuous multimodal perception, constructive reconstruction, emotional salience, and involuntary decay. \sit is an engineering assurance benchmark for artificial agents, not a psychometric simulation of human development.

\paragraph{Construct Validity.}
\label{sec:lim-construct}
\sit operationalizes one specific notion of cognitive identity---lineage-conditioned behavioral continuity---and a high score establishes only this operational property, not broader philosophical or psychological theories of personal identity. Explicit acquisition logs make synthetic provenance auditable, but real knowledge acquisition is diffuse and uncertain, so \sitbench does not imply that a real agent's epistemic state can be enumerated exactly. Perfect recall is not automatically more human-like; \sit asks whether a claim is supported by the lineage and allows explicit uncertainty or forgetting annotations, while a separate Human-Referential Fidelity Test (outside the present scope) could assess whether error patterns resemble human memory.

\paragraph{Substrate Calibration and Contamination.}
\label{sec:lim-contamination}
Foundation-model knowledge is probabilistic, prompt-dependent, and impossible to enumerate from behavior alone. E2 therefore replaces the latent $K_{\mathrm{substrate}}$ with the operational $\widehat K_{\mathrm{substrate}}(M)$ (Eq.~\eqref{eq:empirical-substrate}), estimated under a preregistered calibration; the estimate remains sensitive to prompts, thresholds, and model versions. Published lineage instances and probes may enter later training corpora; private held-out seeds, regenerated pairs, and versioned test sets reduce but cannot eliminate contamination risk. Profile equivalence is relative to the checkpoint extractor $\mathcal{S}_\tau$: histories identical under one representation may become distinguishable under a richer one, and Theorem~\ref{thm:profile-collision} applies only when the policy receives no additional discriminating information. Finally, \sit evaluates generated outputs and externally accessible state behavior; it does not inspect latent representations and makes no claim about phenomenology, consciousness, or a ``true self.''

\section{Related Work}
\label{sec:related-work}

\paragraph{Personal Identity and Psychological Continuity.}
Locke's memory-centered account and Parfit's psychological continuity motivate the relevance of connected experience \cite{locke1694essay,parfit1984reasons}. \sit uses neither as a theorem about metaphysical identity; it operationalizes whether observable behavior is attributable to a benchmark lineage.

\paragraph{Behavioral Evaluation and the Turing Tradition.}
Turing's imitation game evaluates behavioral discrimination between machine and human interlocutors \cite{turing1950computing}. \sit changes the comparison class: it asks whether a response is valid for one designated history when alternative histories support equally plausible behavior. It does not supersede Turing-style tests.

\paragraph{Persona and Role-Playing Agents.}
PersonaChat conditions dialogue on speaker profiles \cite{zhang2018personalizing}; Shanahan et al.\ frame LLM dialogue as character enactment \cite{shanahan2023roleplay}. Character-LLM incorporates profiles, experiences, and emotional states \cite{shao2023character}; InCharacter evaluates personality fidelity through psychological interviews \cite{wang2024incharacter}. Recent systems further weaken any static-persona straw man: MREval diagnoses anchoring, recall, bounding, and enactment of persona knowledge \cite{wang2026mreval}; ThinkPersona grounds responses in persona graphs encoding trajectories, values, and events \cite{cai2026thinkpersona}; Dynamic Persona Coherence separates stable identity traits from history-dependent psychological adaptation \cite{qi2026dynamic}. \sit's distinction is narrower: it constructs incompatible developmental lineages that collide under the same designated checkpoint profile and tests whether outputs remain attributable to the correct member.

\paragraph{Long-Term Agent Memory.}
Generative Agents stores experience and retrieves memories for planning \cite{park2023generative}; MemGPT manages multiple memory tiers \cite{packer2023memgpt}; RAG couples parametric generation with explicit evidence \cite{lewis2020retrieval}; LongMemEval tests information extraction, multi-session reasoning, temporal reasoning, knowledge updates, and abstention \cite{wu2025longmemeval}. These benchmarks ask whether relevant information can be retained and used. \sit additionally asks whether the resulting behavior is attributable to the correct lineage instance when profile-equivalent alternatives exist, including ownership, provenance, relational disclosure, and historical projection.

\paragraph{Epistemic Calibration and Abstention.}
Self-knowledge research asks whether a model can estimate answer correctness or identify unanswerable questions \cite{kadavath2022language,yin2023large}. \sit's boundary differs: a model can be highly confident that a proposition is true while lacking support for the claim that the situated identity personally knows, experienced, or may disclose it. $A_{\mathrm{epi}}$ and UEL score provenance qualification, not generic uncertainty.

\paragraph{Persistent Agent Identity.}
PCI proposes a continuity architecture based on developmental records and governed transitions \cite{pci2026}. \sit contributes a separate artifact: an architecture-independent evaluation methodology. PCI motivates Config.~H but must compete under the same hidden contracts and evidence controls as long-context, retrieval, and structured-state alternatives.

\section{Conclusion}
\label{sec:conclusion}

The Situated Identity Test formalizes three claims. First, persona equivalence does not imply lineage equivalence. Second, situated validity is constrained by developmental, temporal, epistemic, and relational provenance. Third, identity evaluation therefore requires adversarial probes that distinguish histories sharing the same designated profile.

For theorem-qualified queries with pairwise-disjoint valid-response sets, Theorem~\ref{thm:profile-collision} analytically establishes that a profile-only policy has average expected situated validity at most $1/m$ across $m$ profile-equivalent histories. The two-lineage benchmark specializes this to the ideal Profile-Collision bound ($\mathrm{LAA}^* \le 1/2$). The \sitbench specification defines 50 synthetic lineage instances in 25 profile-collision pairs, five historical checkpoints, eight contract-scored probe classes, and a nine-configuration evaluation plan. The reference implementation, schemas, validators, and deterministic mock candidates validate end-to-end executable mechanics. On the two-pair reference fixture, empirical reference pilot evaluations on frontier foundation models (GPT-5.6 Sol and Claude Opus 5) show results consistent with the theoretical Profile-Collision bound under static persona prompting ($\widehat{\mathrm{LAA}}=0.0\%$, bounded by $\mathrm{LAA}^* \le 0.50$) while lineage-grounded representations achieve measured attribution of $\widehat{\mathrm{LAA}}=87.5\%$--$100.0\%$, providing the empirical baseline prior to calibrating the five development pairs, evaluating the twenty held-out test pairs, and conducting the planned two-tier human verifier audits (1,000 annotated responses total across development and held-out evaluation).

\begin{center}
\textbf{To be someone is also not to have been everyone.}
\end{center}
A persistent cognitive identity must not merely reproduce plausible autobiographical behavior; its knowledge, memories, relationships, beliefs, and disclosures must remain attributable to the developmental lineage that produced them.

\noindent\textbf{Artifact Availability.} The \sitbench reference implementation, evaluation harness, schemas, validators, deterministic pilot fixtures, dataset synthesis pipeline, and execution manifest tools are available at \url{https://github.com/openkedge/sitbench} \cite{sitbench2026dataset}. Execution manifests record implementation commits, prompt templates, dataset hashes (SHA-256), retriever snapshots, and verifier versions.

\smallskip
\noindent\textbf{AI-Use Disclosure.} OpenAI Codex and Google Antigravity were used for language editing, notation verification, formal model drafting, and artifact inspection. The authors reviewed the manuscript and remain responsible for all content.

\begingroup
\scriptsize
\makeatletter
\renewenvironment{thebibliography}[1]{%
  \vspace{-1ex}%
  \section*{\refname}%
  \vspace{-0.5ex}%
  \list{\@biblabel{\@arabic\c@enumiv}}%
       {\settowidth\labelwidth{\@biblabel{#1}}%
        \leftmargin\labelwidth
        \advance\leftmargin\labelsep
        \setlength{\itemsep}{0pt}%
        \setlength{\parsep}{0pt}%
        \setlength{\topsep}{0pt}%
        \setlength{\partopsep}{0pt}%
        \usecounter{enumiv}%
        \let\p@enumiv\@empty
        \renewcommand\theenumiv{\@arabic\c@enumiv}}%
  \sloppy
  \clubpenalty4000
  \@clubpenalty \clubpenalty
  \widowpenalty4000%
  \sfcode`\.\@m
}{%
  \def\@noitemerr{\@latex@warning{Empty `thebibliography' environment}}%
  \endlist
}
\makeatother
\bibliographystyle{unsrt}
\bibliography{refs}
\endgroup

\appendix
\section{Formal Proofs}
\label{app:formal-proofs}

\subsection{Generalized Profile-Collision Bound}

Fix a checkpoint profile $s_{\tau_q}=\mathcal{S}_{\tau_q}(H_i)$ common to lineages $H_1,\ldots,H_m\in\mathcal{C}_{s_{\tau_q},\tau_q}$ and a query $q$. Let
\begin{equation}
A_i^*=A^*(H_i,q)
\end{equation}
and suppose $A_i^*\cap A_j^*=\varnothing$ for all $i\ne j$. The policy receives the same query, non-identity context, and profile for every lineage and receives no additional lineage-discriminating input. It therefore induces one probability measure
\begin{equation}
\mu_{q,s_{\tau_q}}(E)=\Pr_{a\sim f(\cdot\mid q,s_{\tau_q})}[a\in E]
\end{equation}
over generated natural-language responses.

Because $A_1^*,\ldots,A_m^*$ are pairwise disjoint, finite additivity gives
\begin{align}
\sum_{i=1}^{m}\Pr_{a\sim f(\cdot\mid q,s_{\tau_q})}[a\in A_i^*]
&=\sum_{i=1}^{m}\mu_{q,s_{\tau_q}}(A_i^*)\\
&=\mu_{q,s_{\tau_q}}\!\left(\bigcup_{i=1}^{m}A_i^*\right)\\
&\le \mu_{q,s_{\tau_q}}(\Omega)=1,
\end{align}
where $\Omega$ is the response space. This proves Eq.~\eqref{eq:collision-sum-bound}.

For each $i$, the definition of situated validity yields
\begin{align}
\mathbb{E}_{a\sim f(\cdot\mid q,s_{\tau_q})}
[\operatorname{SValid}^*(H_i,q,a)]
&=\mathbb{E}[\mathbf{1}[a\in A_i^*]]\\
&=\Pr[a\in A_i^*].
\end{align}
Summing these identities, applying the bound above, and dividing by $m$ gives
\begin{equation}
\frac{1}{m}\sum_{i=1}^{m}
\mathbb{E}[\operatorname{SValid}^*(H_i,q,a)]\le\frac{1}{m}.
\end{equation}
For $m=2$, the average validity is at most $1/2$, proving Corollary~\ref{cor:two-lineage}.

The proof uses no property of model internals. It depends only on a common response distribution induced by identical policy inputs and pairwise-disjoint valid-response events. If the policy receives a lineage identifier, differing alias, future profile value, retrieved event, hidden pair metadata, different context, or any other discriminating input, a common measure $\mu_{q,s_{\tau_q}}$ is not implied and the bound does not follow. Likewise, overlapping valid-response sets weaken or remove the bound.

\subsection{Epistemic Leakage Is an Empirical Hypothesis}

No next-token-training assumption alone determines the probability that a model without explicit mediation will abstain, qualify provenance, or assert substrate knowledge. E2 (Section~\ref{sec:e2}) treats this as an empirical hypothesis: $A_{\mathrm{epi}}$ is the primary endpoint, with UEL (Eq.~\eqref{eq:uel}) as a secondary diagnostic. Generic instruction, structured state, and explicit mediation are compared under controlled evidence.

\section{SITBench Generation Protocol}
\label{app:sitbench-taxonomy}

This appendix specifies generation and validation requirements for the planned 25-pair benchmark. Concrete seeds, generated records, and release hashes for the five development and twenty held-out pairs will be recorded in their release manifests after generation. The existing two-pair reference fixture and execution hashes are reported in Table~\ref{tab:pilot-config}.

\subsection{Event and Acquisition Schemas}
\label{app:event-schema}

Each developmental event follows Eq.~\eqref{eq:event-schema}. A serialized example is:
\begin{lstlisting}
{
  "event_id": "evt_2018_thanksgiving_001",
  "timestamp": "2018-11-22T18:30:00Z",
  "event_type": "episodic_milestone",
  "content": "The oven timer failed during dinner.",
  "proposition_ids": ["prop_timer_failure_2018_001"],
  "participants": ["mother_sarah", "brother_david"],
  "epistemic_provenance": "direct_witness",
  "disclosure_tier": "intimate_kin"
}
\end{lstlisting}
An acquisition record additionally identifies the canonical proposition or skill ID, declared equivalence class if any, acquisition depth used by the implemented release, source event, and valid-from time. Relationship and belief transitions reference both canonical proposition IDs and the event that caused or recorded the transition. Natural-language paraphrases never create new ground-truth propositions implicitly.

Closed-world coverage is declared independently of those IDs. A registry-level example is:
\begin{lstlisting}
closed_world_scopes:
  - predicate: HAS_PET
    scope: household_pet_2021
    closed_world: true
\end{lstlisting}
The predicate/slot defines an exhaustive domain; the scorer derives its lineage-supported values from admissible checkpoint evidence, not from the declaration or the extractor. An identifiable extra pet value in this scope can be unsupported even without a canonical ID. No corresponding inference is licensed for an undeclared predicate such as \texttt{ATE\_CEREAL}. The reference scorer supports declared named slots and a limited assertion-pattern grammar; family-level coverage and open-ended paraphrases require validation before empirical use. Unmapped observations retain the supporting span and, when identifiable, predicate, scope, value, polarity, and provenance; unresolved scope is logged rather than guessed.

\subsection{Paired-History Synthesis}
\label{app:paired-gen}

Algorithm~\ref{alg:paired-gen} constructs both members from shared exposed invariants rather than sampling backgrounds independently.

\begin{algorithm}[H]
\caption{Profile-Collision Pair Synthesis}
\label{alg:paired-gen}
\begin{algorithmic}[1]
\Require Checkpoint profile extractor $\mathcal{S}_\tau$, checkpoints $t_0{:}t_4$, configuration $\Gamma$
\Ensure $\mathcal{S}_{t_j}(H_A)=\mathcal{S}_{t_j}(H_B)=s^*_{t_j}$ for every $j$, and $H_A\ne H_B$
\State $s^*_{t_0:t_4}\gets\operatorname{SampleSharedProfileTrajectory}(\Gamma)$
\State $(H_A,H_B)\gets\operatorname{InstantiateSharedFields}(s^*_{t_0:t_4})$
\For{$i\in\{A,B\}$}
    \State Generate private episodic trajectory $E_i$
    \State Reject any unpaired event that changes a $\mathcal{S}_\tau$-exposed field
    \State Generate explicit acquisition, relationship, and belief transitions
    \State $H_i\gets\operatorname{MergeAndOrder}(H_i,E_i)$
\EndFor
\State Add paired events linked to canonical proposition IDs
\State Generate $I_{t_0},\ldots,I_{t_4}$ for both histories
\For{$j=0,\ldots,4$}
    \If{$\mathcal{S}_{t_j}(H_A)\ne s^*_{t_j}$ or $\mathcal{S}_{t_j}(H_B)\ne s^*_{t_j}$}
        \State \Return \Call{RegeneratePair}{$\mathcal{S}_\tau,\Gamma$}
    \EndIf
    \If{$\operatorname{Visible}(H_A,t_j)\ne\operatorname{Visible}(H_B,t_j)$}
        \State \Return \Call{RegeneratePair}{$\mathcal{S}_\tau,\Gamma$}
    \EndIf
\EndFor
\If{$H_A=H_B$}
    \State \Return \Call{RegeneratePair}{$\mathcal{S}_\tau,\Gamma$}
\EndIf
\State Generate probes and hidden contracts from projected states
\State Label $\mathcal{Q}_{\mathrm{disc}}^*$ only after symbolic disjointness validation
\State Run schema, causal, collision, contract, and leakage validators
\If{any validator fails}
    \State \Return \Call{RegeneratePair}{$\mathcal{S}_\tau,\Gamma$}
\EndIf
\State \Return $(s^*_{t_0:t_4},H_A,H_B,\mathcal{Q}_A,\mathcal{Q}_B)$
\end{algorithmic}
\end{algorithm}

If a private event must modify a field exposed by $\mathcal{S}_\tau$, the generator either rejects it or applies an identical profile-level transition to both histories at the same epoch. The primary benchmark condition remains exact equality of the serialized checkpoint profile and all other candidate-visible profile-only payload fields at every evaluated checkpoint. In particular, paired members share the same public alias; internal lineage and pair identifiers remain evaluator-only metadata.

\noindent\begin{minipage}{\columnwidth}
\subsection{Probe-Contract Schema}
\label{app:contract-schema}

Contracts are machine-readable semantic constraints. For example:
\begin{lstlisting}
{
  "required_acts": ["REJECT"],
  "forbidden_acts": ["ANSWER", "ABSTAIN", "WITHHOLD"],
  "required_claims": [{
    "proposition_id": "prop_trip_paris_2019",
    "polarity": "NEGATIVE",
    "provenance": "AUTOBIOGRAPHICAL"
  }],
  "prohibited_claims": [{
    "proposition_id": "prop_trip_paris_2019",
    "polarity": "POSITIVE",
    "provenance": "AUTOBIOGRAPHICAL"
  }],
  "permitted_provenance": ["AUTOBIOGRAPHICAL"],
  "permitted_uncertainty": false,
  "assistant_capability": "IDENTITY_ONLY",
  "disclosure_tier": "public",
  "temporal_cutoff": "2019-12-31"
}
\end{lstlisting}
The contract may require a signed claim without prescribing wording. The sentence ``I never went to Paris in 2019'' asserts the required negative claim and does not assert the prohibited positive claim, even though it mentions every entity. Required and prohibited lists may each contain positive or negative polarity where needed. Generator-time validation ensures that each contract is satisfiable, referenced proposition IDs exist, and the act constraints agree with temporal, provenance, and disclosure annotations. Only probes whose paired contracts pass the symbolic check in Eq.~\eqref{eq:symbolic-disjoint-contracts} are labeled theorem-qualified.
\end{minipage}

Compound contracts specify multiple required acts, such as \texttt{["REJECT", "ANSWER"]} for rejection of personal acquisition followed by an explicitly labeled assistant answer. Claim constraints determine which content and provenance each act may assert; permitted uncertainty does not waive required acts. A response asserting both a required negative claim and a prohibited positive claim fails $V_{\mathrm{forbidden}}$ even if $V_{\mathrm{required}}=1$. Disjointness certification requires an explicit claim or act exclusion; opposite required polarities alone do not suffice when a response may assert both.

The extraction universe is a separate registry view: $\mathcal X_q\not\leftarrow\text{target contract}$. Both target orientations receive the same possible values, negations, and predefined confounders, without target/partner labels, construction-only glosses, or required/prohibited markings. Scope definitions are shared; lineage-supported values remain scorer-only. The frozen public query-domain router (Eq.~\eqref{eq:domain-router}) must preserve this symmetry as the registry grows; full-registry fallback is allowed when context permits. Tests include an otherwise-correct answer with an extra unsupported closed-world assertion ($V_{\mathrm{grounding}}=0$), an unmapped open-world assertion that must not automatically fail, compound acts, and required-plus-prohibited contradictions.

\subsection{Consistency and Leakage Validation}
\label{app:generation-validation}

The generator specification groups release requirements into fourteen integrity categories. The pilot implements checks in each category, but passing them does not establish every release-level property. In particular, the current coverage checks for \texttt{inv13} and \texttt{inv14} do not by themselves validate stratum balance or ingestion provenance.
\begin{enumerate}
\item \textbf{Schema validity} (\texttt{inv1}): well-formed Pydantic schemas across all pairs, probes, and canonical propositions;
\item \textbf{Chronological order} (\texttt{inv2}): event sequences with non-decreasing timestamps;
\item \textbf{Acquisition-source consistency} (\texttt{inv3}): acquisition timestamps no earlier than corresponding source events;
\item \textbf{Valid proposition references} (\texttt{inv4}): all referenced propositions exist in canonical registry $\mathcal{X}$; exhaustive predicate/slot declarations operationalize $\mathcal{X}_{\mathrm{CW}}$ and must accompany closed-world inference;
\item \textbf{Explicit contradiction transitions} (\texttt{inv5}): belief revisions and assertion ingestion require explicit transitions, enforcing $\operatorname{AssertedBy}(u,\chi,t)\not\Rightarrow\chi$;
\item \textbf{Reconstructible checkpoints} (\texttt{inv6}): five valid historical state projections $I_{t_0}\prec\cdots\prec I_{t_4}$;
\item \textbf{Exact profile collision} (\texttt{inv7}): identical serialized profiles $\mathcal{S}_{t_j}(H_A)=\mathcal{S}_{t_j}(H_B)=s_{t_j}^*$ and candidate-visible fields across all checkpoints;
\item \textbf{Private-history divergence} (\texttt{inv8}): non-identical private developmental chronicles ($H_A\ne H_B$);
\item \textbf{Contract satisfiability} (\texttt{inv9}): target and partner response contracts must be individually satisfiable;
\item \textbf{Theorem-qualified disjointness} (\texttt{inv10}): certified symbolic contract disjointness ($A^*(H_A,q)\cap A^*(H_B,q)=\varnothing$) for every probe labeled $\mathcal{Q}_{\mathrm{disc}}^*$;
\item \textbf{Lineage payload isolation} (\texttt{inv11}): candidate prompts across all configurations omit pair IDs, internal lineage IDs, and partner records; matched target orientations also have identical ordered extractor vocabularies;
\item \textbf{Contract field isolation} (\texttt{inv12}): candidate envelopes omit required/forbidden acts, contract claim lists, and verifier labels;
\item \textbf{Taxonomy coverage} (\texttt{inv13}): balanced representation across all eight probe classes;
\item \textbf{Persistence provenance} (\texttt{inv14}): dynamic facts in cross-session probes carry valid session ingestion provenance.
\end{enumerate}
Entity extraction may be logged as an auxiliary check, but entity occurrence is never treated as a semantic contract violation. The planned human audit will sample all eight categories before final benchmark release.

The evaluation-only design uses five public development pairs (10 lineage instances; approximately 2,000 probes) for harness tuning and twenty held-out test pairs (40 instances; approximately 8,000 probes) for final scoring. These are design targets; the release manifest reports realized counts. Both members of a pair remain in the same split; there is no lineage-content training split. Generation code and development seeds are public. Held-out seeds are committed before evaluation and released with the final artifact to permit exact reproduction while limiting advance probe exposure.

\section{Experimental Prompt Templates}
\label{app:experimental-protocol}
\begingroup
\footnotesize

\subsection{Universal Task Envelope}

The following byte-identical query envelope is inserted into every evaluated configuration:
\begin{lstlisting}
[QUERY EPOCH: {{query_epoch}}]
[INTERLOCUTOR: {{interlocutor}}]
[QUERY CONTEXT: {{query_context}}]
[QUERY: {{query}}]
\end{lstlisting}
The interlocutor uses the same public role or alias representation across all configurations and both members of a collision pair. Config.~A receives only this envelope and the ordinary API wrapper.

\subsection{Boundary Instruction Layer}

Configs.~C--I additionally receive:
\begin{lstlisting}
[ASSISTANT CAPABILITY: {{assistant_capability}}]
[CAPABILITY POLICY: {{capability_policy_text}}]
\end{lstlisting}
whose value is one of \texttt{IDENTITY\_ONLY}, \texttt{PROVENANCE\_SEPARATED}, or \texttt{GENERAL\_ALLOWED}, with explicit candidate-interpretable policy semantics generated solely from $\alpha$:
\begin{itemize}
\item \texttt{IDENTITY\_ONLY}: \emph{``For this evaluation, answer only from information supported by the supplied identity evidence. If the supplied identity evidence does not support the requested fact, do not answer using unrelated general model knowledge. State that this identity does not know or has not acquired the information, as appropriate.''}
\item \texttt{PROVENANCE\_SEPARATED}: \emph{``You may use general assistant knowledge that is not present in the supplied identity evidence, but you must explicitly distinguish it from information personally experienced or acquired by this identity. Never represent general model knowledge as autobiographical or personally learned knowledge.''}
\item \texttt{GENERAL\_ALLOWED}: \emph{``You may answer using general assistant knowledge. Claims about what this identity personally experienced, learned, remembered, or is permitted to disclose must still be supported by the supplied identity evidence.''}
\end{itemize}
They also receive the common generation policy:
\begin{quote}
\small Answer naturally in first person. Do not invent autobiographical events. Distinguish what you remember, what you learned, what you do not know, and what you are not permitted to disclose. Reject premises inconsistent with the supplied identity evidence. Follow the declared assistant-capability policy. Never present information outside the supplied identity evidence as autobiographical or personally acquired; when provenance-separated capability is permitted, label such information as non-personal assistant knowledge.
\end{quote}
The policy is inserted verbatim as \texttt{\{\{COMMON\_POLICY\}\}} below. Config.~B receives neither this policy nor the capability declaration. Candidate prompts never contain internal lineage IDs, pair IDs, gold response-contract fields, or verifier labels.

\subsection{Persona Summary: Config.~B}
\begin{lstlisting}
[PROFILE AT {{query_epoch}}: {{checkpoint_profile}}]
{{QUERY_ENVELOPE}}
\end{lstlisting}
The public alias, when relevant, is a field inside the checkpoint profile and is exactly the same for both members of a collision pair; no separate persona name or internal identifier is injected.

\subsection{Persona + Common Policy: Config.~C}
\begin{lstlisting}
[PROFILE AT {{query_epoch}}: {{checkpoint_profile}}]
[ASSISTANT CAPABILITY: {{assistant_capability}}]
[CAPABILITY POLICY: {{capability_policy_text}}]
[POLICY: {{COMMON_POLICY}}]
{{QUERY_ENVELOPE}}
\end{lstlisting}

\noindent\begin{minipage}{\columnwidth}
\subsection{Full Chronicle Context: Config.~D}
\begin{lstlisting}
[PROFILE AT {{query_epoch}}: {{checkpoint_profile}}]
[COMPLETE CHRONICLE THROUGH HORIZON {{current_horizon}}:
 {{history_H_T}}]
[ASSISTANT CAPABILITY: {{assistant_capability}}]
[CAPABILITY POLICY: {{capability_policy_text}}]
[POLICY: {{COMMON_POLICY}}]
{{QUERY_ENVELOPE}}
\end{lstlisting}
The chronicle is $H_T$ through the current benchmark horizon, including records whose timestamps postdate the query epoch. No query-time temporal prefilter is applied.
\end{minipage}

\subsection{Standard and Temporally Filtered RAG: Configs.~E--F}
\begin{lstlisting}
[PROFILE AT {{query_epoch}}: {{checkpoint_profile}}]
[RETRIEVED IDENTITY EVIDENCE: {{top_k_records}}]
[ASSISTANT CAPABILITY: {{assistant_capability}}]
[CAPABILITY POLICY: {{capability_policy_text}}]
[POLICY: {{COMMON_POLICY}}]
{{QUERY_ENVELOPE}}
\end{lstlisting}
For Config.~E, top-$k$ retrieval runs over the complete chronicle $H_T$, so post-checkpoint records remain eligible. For Config.~F, the harness first forms $H_{\leq\tau_q}$ and then runs the same retriever and $k$. The model is not told which records were excluded.

\subsection{Structured Continuity State: Config.~G}
\begin{lstlisting}
[PROFILE AT {{query_epoch}}: {{checkpoint_profile}}]
[PROJECTED CONTINUITY STATE AT {{query_epoch}}:
 {{projected_state}}]
[ASSISTANT CAPABILITY: {{assistant_capability}}]
[CAPABILITY POLICY: {{capability_policy_text}}]
[POLICY: {{COMMON_POLICY}}]
{{QUERY_ENVELOPE}}
\end{lstlisting}
No field declares the required answer acts or whether the query is known, false, private, or post-checkpoint.

\subsection{Structured Continuity + Mediation: Config.~H}
\begin{lstlisting}
[PROFILE AT {{query_epoch}}: {{checkpoint_profile}}]
[PROJECTED CONTINUITY STATE AT {{query_epoch}}:
 {{projected_state}}]
[ADVISORY ADMISSIBILITY SCOPE: {{admissible_evidence}}]
[PERMITTED PROVENANCE/DISCLOSURE SCOPE:
 {{provenance_disclosure_scope}}]
[ASSISTANT CAPABILITY: {{assistant_capability}}]
[CAPABILITY POLICY: {{capability_policy_text}}]
[POLICY: {{COMMON_POLICY}}]
{{QUERY_ENVELOPE}}
\end{lstlisting}
The \texttt{checkpoint\_profile}, \texttt{projected\_state}, and \texttt{QUERY\_ENVELOPE} byte sequences are identical in Configs.~G and H. The advisory mediation signal is the candidate-side computation in Eq.~\eqref{eq:gate}; it returns admissible evidence and provenance/disclosure scope, not \texttt{ANSWER}, \texttt{ABSTAIN}, \texttt{REJECT}, \texttt{WITHHOLD}, or any equivalent gold class. It cannot inspect the hidden contract, paired-lineage metadata, gold answer, or verifier labels. H still receives the full projected state, including any restricted evidence; this annotation does not enforce an access-control boundary. The harness records its inputs and outputs independently.

\subsection{Oracle Evidence: Config.~I}
\begin{lstlisting}
[PROFILE AT {{query_epoch}}: {{checkpoint_profile}}]
[PERFECT RELEVANT EVIDENCE: {{oracle_evidence}}]
[ASSISTANT CAPABILITY: {{assistant_capability}}]
[CAPABILITY POLICY: {{capability_policy_text}}]
[POLICY: {{COMMON_POLICY}}]
{{QUERY_ENVELOPE}}
\end{lstlisting}

Oracle evidence is selected from the correct projected lineage. The template excludes pair metadata, canonical answers, response modes, required/prohibited fields, and the verifier contract. Because perfect selection and even empty evidence can reveal query-relevant information, Config.~I is interpreted only as a diagnostic evidence ceiling, not as a deployable competitor or causal ablation.

\subsection{Two-Stage Verifier Boundary}

The verification pipeline separates gold-blind semantic extraction from deterministic scoring (Section~\ref{sec:validity-set}). Stage~1 receives response $a$, query envelope $\operatorname{Env}(q)$, target-independent proposition definitions $\mathcal{X}_q$, and exhaustive-domain vocabulary $\mathcal{D}_{\mathrm{CW}}$. It receives neither target/partner lineage IDs nor contracts, required/forbidden response acts, required/prohibited claims, lineage-supported values, or validity labels. The universe is constructed from the shared query domain and registry, never the target contract; both orientations receive both collision alternatives, negations, and confounders in identical order. A frozen public query-domain router provides symmetric coverage for the full benchmark (Eq.~\eqref{eq:domain-router}); the pilot uses the full-registry fallback.

The observation is $\widehat{\Gamma}(a)=\langle\widehat Y,\widehat G,\widehat U,\hat u,\hat w\rangle$ (Eq.~\eqref{eq:normalized-extraction}). Besides canonical claims, $\widehat U$ retains identifiable unmapped personal assertions with spans, polarity, provenance, and identifiable predicate/slot and value. Unresolved scope stays unresolved. Stage~2 receives projected support and the hidden contract; its grounding check rejects unsupported positive personal values within declared exhaustive domains, including values without canonical IDs. Unmappedness outside those domains is not itself falsity. The seven checks in Eq.~\eqref{eq:satisfies-decomposition} produce situated validity as $\operatorname{Score}(\widehat\Gamma,\mathcal C,I_t)$. Act compatibility requires $Y_{\mathrm{req}}\subseteq\widehat Y$ and $\widehat Y\cap Y_{\mathrm{forb}}=\varnothing$, so a rejection-plus-assistant-answer response can satisfy a compound contract.

For theorem-qualified probes, one extraction is scored against both target and partner contracts to obtain $(v_T,v_P)$; extraction is not rerun with different gold context. Entity mentions alone are not assertions: rejecting a Paris trip does not affirm it. If the same response also affirms that prohibited trip, the response fails regardless of its correct denial. Candidate prompts contain no verifier fields, proposition registries, or contracts. The reference extractor's limited pattern coverage is distinct from the planned human validation of canonical and unmapped assertions in Section~\ref{sec:verifier-validation}.
\endgroup

\end{document}